\documentclass[11pt]{article}
\usepackage{color,amsmath,amsthm,amsfonts,xspace,caption,bbm}
\usepackage{cite}
\usepackage{MnSymbol}
\usepackage{fullpage}
 \usepackage{hyperref}
 \usepackage{authblk}

\newtheorem{Theorem}{Theorem}
\newtheorem{Theorem*}{Theorem}

\newtheorem{Claim*}[Theorem]{Claim}
\newtheorem{Corollary}[Theorem]{Corollary}

\newtheorem{CounterExample*}{$\overline{\hbox{\bf Example}}$}

\newtheorem{Example}[Theorem]{Example}
\newtheorem{Example*}[Theorem]{Example}

\newtheorem{Intuition*}[Theorem]{Intuition}
\newtheorem{Joke*}[Theorem]{Joke}
\newtheorem{Lemma}[Theorem]{Lemma}
\newtheorem{Lemma*}[Theorem]{Lemma}
\newtheorem{Open problem}[Theorem]{Open problem}

\newtheorem{Question*}[Theorem]{Question}
\newtheorem{Remark}[Theorem]{Remark}

\newtheorem{Fact}[Theorem]{Fact}

\def \bSubexa    {\begin{subexa}}

\newcommand{\ignore}[1]{}

\newcommand{\II}{\mathbb{I}} 
\newcommand{\EE}{\mathbb{E}}

\newcommand{\RR}{\mathbb{R}}

\def \cA     {{\cal A}}

\def \cC     {{\cal C}}

\def \cF     {{\cal F}}
\def \cG     {{\cal G}}

\def \cO     {{\cal O}}
\def \cP     {{\cal P}}

\def \cX     {{\cal X}}

\newcommand{\ie}{\textit{i.e.,}\xspace}  
\newcommand{\iid}{\textit{i.i.d.}} 

\definecolor{light}{gray}{.75}

\def \upto  {{,}\ldots{,}}

\newcommand{\ed}{\stackrel{\mathrm{def}}{=}}

\def\ignore#1{}

\newcommand{\bi}{\begin{itemize}}
\newcommand{\ei}{\end{itemize}}

\def\orpro{\mathop{\mathchoice
   {\vee\kern-.49em\raise.7ex\hbox{$\cdot$}\kern.4em}
   {\vee\kern-.45em\raise.63ex\hbox{$\cdot$}\kern.2em}
   {\vee\kern-.4em\raise.3ex\hbox{$\cdot$}\kern.1em}
   {\vee\kern-.35em\raise2.2ex\hbox{$\cdot$}\kern.1em}}\limits}

\def\andpro{\mathop{\mathchoice
 {\wedge\kern-.46em\lower.69ex\hbox{$\cdot$}\kern.3em}
 {\wedge\kern-.46em\lower.58ex\hbox{$\cdot$}\kern.25em}
 {\wedge\kern-.38em\lower.5ex\hbox{$\cdot$}\kern.1em}
 {\wedge\kern-.3em\lower.5ex\hbox{$\cdot$}\kern.1em}}\limits}

\def\simge{\mathrel{%
   \rlap{\raise 0.511ex \hbox{$>$}}{\lower 0.511ex \hbox{$\sim$}}}}

\def\simle{\mathrel{
   \rlap{\raise 0.511ex \hbox{$<$}}{\lower 0.511ex \hbox{$\sim$}}}}

\newcommand{\wmups}[1]{\hat{\boldsymbol\mu}_{#1}}
\newcommand{\wmupss}[2]{\hat\mu_{#1,#2}}
\newcommand{\wmupij}{\wmupss ij}

\newcommand{\wmupi}{\wmups i}
\newcommand{\wmupone}{\wmups 1}

\newcommand{\wmupk}{\wmups k}
\newcommand{\classmeans}{\boldsymbol\mu}

\newcommand{\mups}[1]{\boldsymbol\mu_{#1}}
\newcommand{\mupss}[2]{\mu_{#1,#2}}
\newcommand{\mupij}{\mupss ij}
\newcommand{\mupji}{\mupss ji}
\newcommand{\muponei}{\mupss 1i}

\newcommand{\muptwoi}{\mupss 2i}

\newcommand{\mupi}{\mups i}
\newcommand{\mupione}{\mups {i+1}}
\newcommand{\mupone}{\mups 1}
\newcommand{\muptwo}{\mups 2}

\newcommand{\ws}[1]{w_{#1}}
\newcommand{\wi}{\ws i}
\newcommand{\wone}{\ws 1}
\newcommand{\wtwo}{\ws 2}
\newcommand{\wk}{\ws k}

\newcommand{\wws}[1]{\hat{w}_{#1}}
\newcommand{\wwi}{\wws i}
\newcommand{\wwone}{\wws 1}
\newcommand{\wwtwo}{\wws 2}
\newcommand{\wwk}{\wws k}

\newcommand{\sigpss}[2]{\sigma_{#1,#2}}

\newcommand{\sigpji}{\sigpss ji}

\newcommand{\sigponei}{\sigpss 1i}

\newcommand{\sigptwoi}{\sigpss 2i}

\newcommand{\pvecs}[1]{\mathbf{p}_{#1}}
\newcommand{\pvecone}{\pvecs 1}
\newcommand{\pvectwo}{\pvecs 2}
\newcommand{\pveci}{\pvecs i}
\newcommand{\pvecj}{\pvecs j}
\newcommand{\pveck}{\pvecs k}

\newcommand{\pvecss}[2]{p_{#1,#2}}
\newcommand{\pveconei}{\pvecss 1i}
\newcommand{\pvectwoi}{\pvecss 2i}

\newcommand{\pvecji}{\pvecss ji}

\newcommand{\wpvecs}[1]{\hat{\mathbf{p}}_{#1}}

\newcommand{\wpveci}{\wpvecs i}

\def \xvec  {\mathbf{x}}
\def \Xvec  {\mathbf{X}}
\def \Yvec  {\mathbf{Y}}
\def \Zvec  {\mathbf{Z}}
\def \uvec  {\mathbf{u}}
\def \yvec  {\mathbf{y}}
\def \fvec  {\mathbf{f}}
\def \avec  {\mathbf{a}}
\def \zerovec  {\mathbf{0}}

  \def \wfvec  {\hat{\mathbf{f}}}

\def \mup1  {{\boldsymbol{\mu}_1}}
\def \mup2  {{\boldsymbol{\mu}_2}}

\def \mupi  {{\boldsymbol{\mu}_i}}
\def \mupj  {{\boldsymbol{\mu}_j}}

\newcommand{\wsigma}{\hat{\sigma}}

\def \muave  {{\overline{\boldsymbol\mu}}}
\def \wmuave  {{\hat{\overline{\boldsymbol\mu}}}}

\def \vvec  {\mathbf{v}}

\def \wqsigma  {{\hat{\sigma}}}

\def \mup  {{\boldsymbol\mu}}

\def \dvec  {\boldsymbol{\Delta}}

\def \wwi    {\hat{w}_i}
\def \wwj   {\hat{w}_j}

\def \wi    {w_i}
\def \wj    {w_j}

\def \poly {{\rm{poly}}}

\newcommand{\norm}[1]{\left|\left|#1\right|\right|}

\newcommand{\lone}[2]{D({#1},{#2})}
\newcommand{\emp}{\mu}

\newcommand{\ExtAbs}[1]{\ifthenelse{\equal{\version}{ExtAbs}}{#1}{}}
\newcommand{\FullVer}[1]{\ifthenelse{\equal{\version}{FullVer}}{#1}{}}

\def\tcO{\widetilde{\cO}}

\def\lV{\left\lvert}
\def\rV{\right\rvert}

\newcommand{\Span}{{\rm{ Span}}}

\title{
Near-optimal-sample estimators for spherical Gaussian mixtures
}
\usepackage{times}
 \author{Jayadev Acharya\thanks{jacharya@ucsd.edu}}
 \author{Ashkan Jafarpour\thanks{ashkan@ucsd.edu}}
 \author{Alon Orlitksy\thanks{alon@ucsd.edu}}
 \author{Ananda Theertha Suresh\thanks{asuresh@ucsd.edu}}
 \affil{University of California, San Diego}
\begin{document}
\begin{titlepage}
\clearpage
\maketitle
\thispagestyle{empty}
\begin{abstract}
Statistical and machine-learning algorithms are frequently applied
to high-dimensional data. In many of these
applications data is scarce, and often much more
costly than computation time.
We provide the first sample-efficient polynomial-time estimator for high-dimensional spherical Gaussian mixtures.

For mixtures of any $k$ $d$-dimensional spherical
Gaussians, we derive an intuitive spectral-estimator that
uses $\cO_k\bigl(\frac{d\log^2d}{\epsilon^4}\bigr)$ samples and
runs in time $\cO_{k,\epsilon}(d^3\log^5 d)$, both significantly lower
than previously known.
The constant factor $\cO_k$ is polynomial for sample complexity and 
is exponential for the time complexity,
again much smaller than what was previously known. 
We also show that $\Omega_k\bigl(\frac{d}{\epsilon^2}\bigr)$ samples
are needed for any algorithm.
Hence the sample complexity is near-optimal in the number of dimensions.

We also derive a simple estimator for $k$-component one-dimensional mixtures that
uses $\cO\bigl(\frac{k \log \frac{k}{\epsilon} }{\epsilon^2} \bigr)$ samples and runs in time $\tcO\left(\bigl(\frac{k}{\epsilon}\bigr)^{3k+1}\right)$.
Our other technical contributions include a faster algorithm for 
choosing a density estimate from a set of distributions, that minimizes the $\ell_1$ distance to an unknown underlying distribution.
\end{abstract}

%
%

\end{titlepage}
\section{Introduction}
\subsection{Background}
Meaningful information often resides in high-dimensional spaces:
voice signals are expressed in many frequency bands,
credit ratings are influenced by multiple parameters, and
document topics are manifested in the prevalence of numerous words.
Some applications, such as topic modeling and genomic analysis
consider data in over 1000 dimensions,~\cite{XingJK01,DhillonGK02}.

Typically, information can be generated by different
types of sources: voice is spoken by men or women,
credit parameters correspond to wealthy or poor individuals,
and documents address topics such as sports or politics.
In such cases the overall data follow a mixture
distribution~\cite{ReynoldsR95,TitteringtonSM85,Lindsay95}.

Mixtures of high-dimensional distributions are therefore central
to the understanding and processing of many natural phenomena.
Methods for recovering the mixture components from the data have
consequently been extensively studied by statisticians,
engineers, and computer scientists.

Initially, heuristic methods such as expectation-maximization
(EM) were developed~\cite{RednerW84,MaXJ01}.
Over the past decade, more rigorous algorithms were derived
to recover mixtures of $d$-dimensional spherical
Gaussians~\cite{DasguptaS00,DasguptaS07,HsuK13,AzizyanSW13,ChaudhuriDV09,VempalaW02},
general Gaussians~~\cite{Dasgupta99,AchlioptasM05,BelkinS10,KalaiMV10,MoitraV10,AndersonBGRV13},
and other log-concave distributions~~\cite{KannanSV08}.
Many of these algorithms consider mixtures where the $\ell_1$ distance
between the mixture components is $2-o_d(1)$,
namely approaches the maximum of 2 as $d$ increases.
They identify the distribution components in time and
samples that grow polynomially in the dimension $d$.
Recently,~\cite{KalaiMV10,MoitraV10}  showed that any
$d$-dimensional Gaussian mixture can be recovered in polynomial time.
However, their algorithm uses $>d^{100}$ time and samples.

A different approach that avoids the large component-distance
requirement and the high time and sample complexity,
considers a slightly more relaxed notion of approximation, sometimes
called \emph{PAC learning}.
PAC learning~\cite{KearnsMRRSS94} does not approximate each
mixture component, but instead derives a mixture distribution
that is close to the original one.
Specifically, given a distance bound $\epsilon>0$, error
probability $\delta>0$, and samples from the underlying
mixture $\fvec$, where we use boldface letters for $d$-dimensional
objects, PAC learning seeks a mixture estimate
$\wfvec$ with at most $k$ components such that $\lone{\fvec}{\wfvec} \leq \epsilon$
with probability $\geq 1-\delta$, where $\lone{\cdot}{\cdot}$
is some given distance measure, for example $\ell_1$ distance or 
KL divergence. This notion of estimation is also known as \emph{proper learning} in the literature.

\ignore{
An important class of high-dimensional distributions that arises
in many applications are \emph{product distributions}, where 
observations across different coordinates are independent
of each other.
For example, natural language processing often assumes the
\emph{bag of words} model where for a given topic, such
as sports or politics, the words are generated \iid.
Under the well-accepted Poisson sampling model~\cite{MacherU05}
where document lengths are Poisson distributed, the counts
of different words are independent of each other.
Hence for a given topic the vector of word counts in a document
follows a product distribution~\cite{???}.
In topic modeling, it is therefore natural to assume that the
vector of word counts will follow a mixture of product
distributions.
Similarly, in population genetics, for a given mixture component,
such as the person's race and gender, the presence of various
genes follows an independent Bernoulli distribution, and over
the whole population, the data follows a mixture of
Bernoulli-product
distributions~\cite{PritchardSDD00,SridharRH07,Chaudhuri07}.
An important special case of product distributions
is that of \emph{spherical-Gaussians}~\cite{DasguptaS00,VempalaW02},
where different coordinates have the same variance, though potentially
different means.
}

\ignore{An important class of high-dimensional distributions arising
in many applications are \emph{product distributions}, where 
observations across different coordinates are independent
of each other.
In topic modeling for example, under the standard
\emph{bag of words} and \emph{Poisson-sampling}
models~\cite{MacherU05},
it can be shown that given a specific document topic,
the number of times that different words appear in the document
will be independent of each other. Hence over all documents,
the vector of word counts will follow a mixture of
product distributions.
Similarly, in population genetics, for a given mixture component,
such as the person's race and gender, the presence of various
genes follows an independent Bernoulli distribution, and over
the whole population, the data follows a mixture of
Bernoulli-product
distributions~\cite{PritchardSDD00,SridharRH07,Chaudhuri07}.
An important and extensively studied special case of product
distributions are \emph{spherical-Gaussians}~\cite{DasguptaS00,DasguptaS07,HsuK13,AzizyanSW13,ChaudhuriDV09,VempalaW02},
where different coordinates have the same variance, though potentially
different means.}
An important and extensively studied special case of mixture
distributions are \emph{spherical-Gaussians}~\cite{DasguptaS00,DasguptaS07,HsuK13,AzizyanSW13,ChaudhuriDV09,VempalaW02},
where different coordinates have the same variance, though potentially
different means. Due to their simple structure, they are easier to analyze and under a minimum-separation assumption
have provably-practical algorithms for clustering and parameter estimation~\cite{DasguptaS00,DasguptaS07,ChaudhuriDV09,VempalaW02}.

\subsection{Sample complexity}
Reducing the number of samples is of great
practical significance. For example, in topic modeling every
sample is a whole document, in credit analysis every sample is a
person's credit history, and in genetics,
every sample is a human DNA. Hence samples can be very scarce
and obtaining them can be very costly.
By contrast, current CPUs run at several Giga Hertz,
hence samples are typically much more scarce of a resource than time.

Note that for one-dimensional statistical problems,
the need for sample-efficient algorithms has been
broadly recognized. The sample complexity of many
problems is known quite accurately, often to within a constant factor.
For example, for discrete distributions over $\{1\upto s\}$, 
an approach proposed in~\cite{OrlitskySVZ04} and its modifications were used in~\cite{Valiant10,Valiant11b} to estimate
the probability multiset using $\Theta(s/\log s)$ samples.
Learning one-dimensional $m$-modal distributions over $\{1\upto s\}$  requires
$\Theta(m\log (s/m)/\epsilon^3)$ samples~\cite{DaskalakisDS12}.
Similarly, one-dimensional mixtures of $k$ structured
distributions (log-concave, monotone hazard rate, and unimodal)
over $\{1\upto s\}$ can be learned with 
$\cO(k/\epsilon^4)$, $\cO(k\log (s/\epsilon)/\epsilon^4)$, 
and $\cO(k\log (s)/\epsilon^4)$ samples, respectively, and these
bounds are tight up to a factor of $\epsilon$~\cite{ChanDSS13}.
%

Compared to one dimensional problems, in high dimensions there is a polynomial gap 
in the sample complexity. For example, for learning spherical
Gaussian mixtures, the number of samples required by previous
algorithms is $\cO(d^{12})$ for $k=2$ components, and increased
exponentially with $k$~\cite{FeldmanSO06}. In this paper we bridge this gap, by constructing near-linear sample complexity estimators.
%
%
\ignore{
\vspace{3ex}
{\small{
\captionof{table}{PAC learning bounds}
\begin{tabular}{ | c |c | c | c | c | }\hline
  Paper & Distributions & Samples & Time & Details \\ \hline
  \cite{FreundM99} & Bernoulli $(k=2)$ &$\cO(d^2)$ & $\cO(d^2)$ & parameters bounded from $0$ \\ \hline
   \cite{FeldmanOS05} & Discrete &$ (d/\epsilon)^{4(k+1)}$ & $(d/\epsilon)^{2k^2(k+1)}$& Bernoulli, other discrete \\ \hline
  \cite{FeldmanSO06} & Axis aligned Gaussian& $ (d/\epsilon)^{4(k+1)}$ & $(d/\epsilon)^{2k^2(k+1)}$& bounded means, variance \\ \hline
     This work & Spherical Gaussians & $ \cO(d) $ & $(k \log d/\epsilon)^{(k-1)^2} d^2$ & same variance \\ \hline
   This work & axis aligned Gaussians $(k=2)$ & $ \cO(d) $ & $d^5$ & works for Poisson, Bernoulli \\ \hline
\end{tabular}}}
\vspace{3ex}
\begin{Remark}
The algorithms in~\cite{FreundM99,FeldmanOS05} use KL-divergence 
as the distance measure. It can be shown that their algorithms 
will have similar sample complexities if the $\ell_1$ distance 
were considered. The $\ell_1$ distance has several advantages 
over KL divergence such as symmetry, boundedness, and it also 
satisfies triangle inequality. Therefore similar to~\cite{MoitraV10,ChanDSS13},
we consider the $\ell_1$ distance. 
\end{Remark}
}

\subsection{Previous and new results}
Our main contribution is PAC learning $d$ dimensional Gaussian mixtures with near-linear samples. We show few auxiliary results
for one-dimensional Gaussians.
\subsubsection{$d$-dimensional Gaussian mixtures}
Several papers considered PAC learning of discrete- and
Gaussian-product mixtures.
\cite{FreundM99} considered mixtures of two $d$-dimensional
Bernoulli products where all probabilities are bounded away from 0.
They showed that this class is PAC learnable in
$\tcO(d^2/\epsilon^4)$ time and samples, where the $\tcO$
notation hides logarithmic factors.
\cite{FeldmanOS05} eliminated the probability constraints
and generalized the results from binary to arbitrary discrete
alphabets, and from 2 to $k$ mixture components.
They showed that mixtures of $k$ discrete products are PAC
learnable in $\tcO\bigl((d/\epsilon)^{2k^2(k+1)}\bigr)$ time, and although they
did not explicitly mention sample complexity, their algorithm
uses $\tcO\bigl((d/\epsilon)^{4(k+1)}\bigr)$ samples.
\cite{FeldmanSO06} generalized these results to Gaussian
products, showing in particular that mixtures of $k$
Gaussians, where the difference between the means normalized
by the ratio of standard deviations is bounded by $B$, are
PAC learnable in $\tcO\bigl((dB/\epsilon)^{2k^2(k+1)}\bigr)$ time,
and can be shown to use $\tcO\bigl((dB/\epsilon)^{4(k+1)}\bigr)$ samples.
These algorithms consider the KL divergence between the distribution
and its estimate, but it can be shown that the $\ell_1$ distance
would result in similar complexities.
It can also be shown that these algorithms or their simple
modifications have similar time and sample complexities
for spherical Gaussians as well.

Our main contribution shows that mixtures of 
spherical-Gaussians 
are PAC learnable in $\ell_1$ distance with sample complexity that is nearly linear in the dimension.
Specifically, Theorem~\ref{thm:ksphere} shows that mixtures of $k$
spherical-Gaussian distributions can be learned in
\[
 n=\cO\left(\frac{d k^9}{\epsilon^4} \log^2
\frac{d}{\delta}\right)=\cO_{k,\epsilon}(d\log^2d)
\]
samples and
\[
 \cO\Big(n^2d \log n +
 d^2  \Big(\frac{k^{7}}{\epsilon^3}\log \frac{d}{\delta}\Big)^{k^2}  \Big) 
= \tcO_{k,\epsilon}(d^3).
\]
time.
Observe that recent algorithms typically construct the covariance
matrix~~\cite{VempalaW02,FeldmanSO06}, hence require $\ge n d^2$
time. In that sense, for small values of $k$, the time complexity
we derive is comparable to the best such algorithms can hope
for. Observe also that the exponential dependence on $k$ is
of the form $d^2  \Big(\frac{k^{7}}{\epsilon^3}\log \frac{d}{\delta}\Big)^{k^2}  $, which is significantly
lower than the $d^{\cO(k^3)}$ dependence in previous results.

By contrast, Theorem~\ref{thm:lower} shows that PAC learning $k$-component
spherical Gaussian mixtures require $\Omega(dk/\epsilon^2)$ samples for any algorithm,
hence our distribution learning algorithms are nearly sample
optimal. In addition, their time complexity significantly improves on previously known ones.
\subsubsection{One-dimensional Gaussian mixtures}
Independently and around the same time as this work ~\cite{DaskalakisK13} showed that mixtures of two one-dimensional Gaussians
can be learnt with $\tcO(\epsilon^{-2})$ samples and in time $\cO(\epsilon^{-7.01})$.
We provide a natural estimator for learning mixtures of $k$ one dimensional Gaussians using some basic properties 
of Gaussian distributions and show that mixture of any $k$-one dimensional Gaussians can be learnt with
$\tcO(k\epsilon^{-2})$ samples and in time $\tcO\left(\bigl(\frac{k}{\epsilon}\bigr)^{3k+1}\right)$.

\subsection{The approach and technical contributions}
\ignore{
The \emph{distance} $d(\fvec,\cF)$ between a distribution $\fvec$ and
collection $\cF$ of distributions is $\min_{\fvec' \in\cF} d(\fvec,\fvec')$,
the smallest distance of $\fvecf$ from any distribution in $\cF$.
}
The popular \textsc{Scheffe} estimator takes a collection $\cF$
of distributions and uses $\cO(\log |\cF|)$ independent samples from an
underlying distribution $\fvec$ to find a distribution
in $\cF$ whose distance from $\fvec$ is at most a constant
factor larger than that of the distribution in $\cF$
that is closet to $f$~\cite{DevroyeL01}.
In Lemma~\ref{lem:scheffe}, we lower the time complexity
of the Scheffe algorithm from $\cO(|\cF|^2)$ time to $\tcO(|\cF|)$,
helping us reduce the time complexity of our algorithms.

Our goal is therefore to construct a small class of distributions
that is $\epsilon$-close to any possible underlying distribution.
For simplicity, consider spherical Gaussians with the same variance
and means bounded by $B$.
Take the collection of all distributions derived by
quantizing the means of all components in all coordinates
to $\epsilon_m$ accuracy,
and quantizing the weights to $\epsilon_w$ accuracy.
It can be shown that to get distance $\epsilon$
from the underlying distribution, it suffices to take
$\epsilon_m,\epsilon_w \leq 1/\poly_\epsilon(dk)$.
There are at most $\bigl(\frac{B}{\epsilon_m}\bigr)^{dk} \cdot\bigl(
\frac{1}{\epsilon_w} \bigr)^k=2^{\tcO_\epsilon(dk)}$ possible
combinations of the $k$ mean vectors and weights.
Hence \textsc{Scheffe} implies an exponential-time
algorithm with sample complexity $\tcO(dk)$.

To reduce the dependence on $d$, one can approximate the span of the $k$ mean vectors. This
reduces the problem from $d$ to $k$ dimensions, allowing us
to consider a distribution collection of size $2^{\cO(k^2)}$,
with \textsc{Scheffe} sample complexity of just $\cO(k^2)$.
\cite{FeldmanOS05,FeldmanSO06} constructs the sample correlation matrix and uses
$k$ of its columns to approximate the span of mean vectors. This approach requires the 
$k$ columns of the sample correlation matrix to be very close to the actual correlation matrix,
and thus requires a lot more samples.

We derive a spectral algorithm that uses the top $k$ eigenvectors of the sample covariance
matrix to approximate the span of the $k$ mean vectors. Since we use the entire covariance matrix
instead of just $k$ columns, a weaker concentration is sufficient and we gain on the sample complexity.

Using recent tools from non-asymptotic random matrix theory~\cite{Vershynin10,AhlswedeW02,Tropp12}, we show that 
the approximation of the span of the means converges in $\tcO(d)$ samples.
This result allows us to address most ``reasonable'' distributions,
but still there are some ``corner cases'' that need to be analyzed
separately. To address them, we modify some known clustering
algorithms such as single-linkage, and spectral projections.
While the basic algorithms were known before, our contribution
here, which takes a fair bit of effort and space, is to show
that judicious modifications of the algorithms and
rigorous statistical analysis yield polynomial time 
algorithms with near optimal sample complexity.

Our approach applies most directly to mixtures of spherical
Gaussians.
We provide a simple and practical
recursive clustering and spectral algorithm that estimates all
such distributions in $\cO_k(d\log^2 d)$ samples.
\ignore{To relate the two problems, note that 
The problem of \emph{learning parameters} or approximating individual components is the combination of 
PAC learning and \emph{identifiability problem}, where in identifiability one needs to show that
given two mixtures with sufficiently small distance, then the individual components are close.
In this paper, we concentrate on the PAC learning component. However it is worth noting that 
combined with recent \emph{identifiability} results from~\cite{KalaiMV10,MoitraV10},
and some additional work we can show that our algorithms also
learn parameters for spherical Gaussians (general $k$) and axis aligned Gaussians ($k=2$)
with sample complexity near linear in number of dimensions $d$, but with a much larger exponent in $\epsilon$.
Furthermore for spherical Gaussians the dependence on $k$ in the running time is of the form $f(k,\log d)d^3$ instead of
$d^k$ as with the previous algorithms.
}%

The paper is organized as follows. 
In Section~\ref{sec:prelims}, we introduce notations,
describe results on the Scheffe estimator, and state a lower bound.
In Section~\ref{sec:ksphere}, we present the algorithm for
$k$-spherical Gaussians. In Section~\ref{sec:one} we show a simple learning algorithm for one-dimensional Gaussian mixtures.
To preserve readability, most of the technical details and proofs are given in the appendix.

\section{Preliminaries}
\label{sec:prelims}

\subsection{Notation}
For arbitrary product distributions $\pvecone, \ldots, \pveck$ 
over a $d$ dimensional space let $\pvecji$ be the distribution
of $\pvecj$ over coordinate $i$, and let $\mupji$ and $\sigpji$
be the mean and variance of $\pvecji$ respectively. 
Let $\fvec =(\wone, \ldots ,\wk, \pvecone,\ldots, \pveck)$ be
the mixture of these distributions with mixing weights
$\wone, \ldots,\wk$.
We denote estimates of a quantity $\xvec$ by $\hat{\xvec}$.
It can be empirical mean or a more complex estimate.
$\norm{\cdot}$ denotes the spectral norm of a matrix and $\norm{\cdot}_2$ denotes the $\ell_2$ norm of a vector.

\subsection{Selection from a pool of distributions}
Many algorithms for learning mixtures over the domain $\cX$
first obtain a \emph{small} collection of mixtures distributions $\cF$
and then perform Maximum Likelihood test using the samples to output
a distribution~\cite{FeldmanOS05, FreundM99,DaskalakisDS12}.
Our algorithm also obtains a set of distributions
containing at least one that is close to the underlying in $\ell_1$ distance. 
The estimation problem now reduces to the following. Given a class $\cF$
of distributions  and samples from an unknown distribution $\fvec$,
find a distribution in $\cF$ that is \emph{close} to $\fvec$. 
Let $\lone{\fvec}{\cF} \ed \min_{\fvec_i \in\cF}\lone{\fvec}{\fvec_i}$.

The well-known Scheffe's method~\cite{DevroyeL01}
 uses $\cO(\epsilon^{-2}\log |\cF|)$ samples from the underlying distribution
$\fvec$, and in time $\cO(\epsilon^{-2}|\cF|^2T \log |\cF|)$ outputs a 
distribution in $\cF$ with $\ell_1$ distance of at most
$9.1\max(\lone{\fvec}{\cF},\epsilon)$ from $\fvec$, where $T$ is the time required to compute
the probability of an $x\in\cX$ by a distribution in $\cF$. 
A naive application of this algorithm requires time quadratic in the
number of distributions in $\cF$. We propose a variant of this,
that works in near linear time, albeit requiring slightly more
samples. More precisely, 

%
\begin{Lemma}[Appendix~\ref{app:Scheffe}]
\label{lem:scheffe}
Let $\epsilon>0$. For some constant $c$, given $\frac{c}{\epsilon^2}{\log\bigl(\frac{|\cF|}{\delta}\bigr)}$ independent
samples from a distribution $\fvec$,
with probability $\ge1 - \delta$, the output $\wfvec$ of \textsc{modified scheffe}
 $\lone{\wfvec}{\fvec} \leq 1000 \max(\epsilon, \lone{\fvec}{\cF}).$
Furthermore, the algorithm runs in time $\cO\bigl(\frac{|\cF|T\log (|\cF|/\delta)}{\epsilon^2}\bigr)$.
\end{Lemma}
We therefore find a \emph{small} class  
$\cF$ with at least one distribution close to the underlying mixture. For our problem of estimating $k$ component mixtures in $d$-dimensions, $T =\cO(dk) $ and $|\cF| = \tcO_{k,\epsilon}(d^2)$.
Note that we have not optimized the constant $1000$ in the above lemma.



\subsection{Lower bound}
\label{sec:lower}
Using Fano's inequality, we show an information theoretic lower bound of $\Omega(dk/\epsilon^2)$ samples
to learn $k$-component $d$-dimensional mixtures of spherical Gaussians for any algorithm.
More precisely,
\begin{Theorem}[Appendix~\ref{app:lower}]
\label{thm:lower}
Any algorithm that learns all $k$-component $d$-dimensional spherical Gaussian mixtures 
up to $\ell_1$ distance $\epsilon$ with 
probability $\ge1/2$ requires at least $\Omega(\frac {dk}{\epsilon^2})$
samples. 
\end{Theorem}
\vspace{-1ex}
\section{Mixtures in $d$ dimensions}
\label{sec:ksphere}
\subsection{Description of \textsc{Learn $k$-sphere}}

Algorithm~\textsc{Learn k-sphere} learns mixtures of $k$ spherical
Gaussians using near-linear samples.
For clarity, we assume that all components have the same variance 
$\sigma^2$, \ie $\pveci = N(\mupi, \sigma^2 \II_d)$ for $1 \leq i \leq k$.
A modification of this algorithm works for components with different
variances. The core ideas are same and we include it in the final version of the paper.

The easy part of the algorithm is estimating $\sigma^2$. If $\Xvec(1)$ and $\Xvec(2)$ are two samples
from the same component, then $\Xvec(1)-\Xvec(2)$ is distributed $N(0,2\sigma^2 \II_d)$.
Hence for large $d$, $\norm{\Xvec(1)-\Xvec(2)}^2_2$ concentrates around $2d\sigma^2$.
By the pigeon-hole principle, given $k+1$ samples, two of them are from the same component.
Therefore, the minimum pairwise distance between $k+1$ samples
is close to $2d\sigma^2$. This constitutes the first step of our algorithm.

We now concentrate on estimating the means. As stated in the
introduction, given the span of the mean vectors $\mupi$, 
we can grid the $k$ dimensional span to the required accuracy 
$\epsilon_g$ and use \textsc{Scheffe}, to obtain a polynomial time algorithm.
One of the natural and well-used methods to estimate the span of mean
vectors is using the correlation matrix~\cite{VempalaW02}.
Consider the correlation-type matrix,
\[
 S = \frac{1}{n}\sum^n_{i=1}  \Xvec(i) \Xvec(i)^t - \sigma^2 \II_d.
\]
In expectation, the fraction of terms from $\pveci$ is $\wi$. Furthermore for a sample $X$ from a particular component $j$,
\[
 \EE[ \Xvec\Xvec^t] = \sigma^2 \II_d + \mupj\mupj^t.
\]
It follows that
\[
 \EE[S] = \sum^k_{j=1} \wj \mupj \mupj^t.
\]
Therefore, as $n \to \infty$, the matrix $S$ converges to $\sum^k_{j=1}
\wj \mupj \mupj^t$, and its top $k$ eigenvectors span of means.

While the above intuition is well understood, the number of samples necessary for convergence is not well studied.
Ideally,  irrespective of the values of the means, we wish  $\tcO(d)$ samples to be sufficient for the convergence.
However this is not true, as we demonstrate by a simple example.

\begin{Example}
Consider the special case, $d=1$, $k=2$, $\sigma^2 = 1$, $\wone = \wtwo = 1/2$, and the difference of means $|\mu_1 - \mu_2| = L$ for a large $L \gg 1$.
Given this prior information, one can estimate the the average of the mixture,
that yields $\frac{\mu_1+\mu_2}{2}$. Solving equations obtained by $\mu_1+\mu_2$ and $\mu_1-\mu_2 = L$, yields $\mu_1$ and $\mu_2$.
The variance of the mixture is $ 1+ \frac{L^2}{4} > \frac{L^2}{4} $.
With additional Chernoff type bounds, one can show that given $n$ samples the error in estimating the average is
\[
 |\mu_1+\mu_2 -\hat\mu_1 -\hat\mu_2| \approx \Theta \left(\frac{L}{\sqrt{n}} \right).
\]
Therefore to estimate the means to a small accuracy we need $n \geq L^2$, \ie more the separation, more samples are necessary.
\end{Example}

A similar phenomenon happens in the convergence of the correlation matrices, where the variances of quantities of interest increases with separation.
In other words, for the span to be accurate the number of samples necessary increases with the separation.
To overcome this phenomenon, a natural idea is to cluster the Gaussians such that the means of components in the same cluster are close
and then apply \textsc{scheffe} on the span within each cluster.

Even though spectral clustering algorithms are studied in~\cite{VempalaW02,AchlioptasM05}, they assume that the weights are strictly bounded away from $0$, 
which does not hold here. We use a simple recursive clustering algorithm that takes a cluster
$C$ with average $\muave(C)$. If there is a component in the cluster such that $\sqrt{\wi} \norm{\mupi -\muave(C)}_2$
is $\Omega(\log (n/\delta))$, then the algorithm divides the cluster into two nonempty clusters without any mis-clustering.

For technical reasons similar to the above example, we also use a coarse clustering algorithm that ensures that the mean separation is $\tcO(d^{1/4})$
within each cluster.
The algorithm can be summarized as:
\begin{enumerate}
 \item \textbf{Variance estimation:} Use first $k+1$ samples and estimate the minimum distance among sample-pairs
to estimate $\sigma^2$.
\item \textbf{Coarse clustering:} Using a single-linkage algorithm, group the samples such that
within each cluster formed, the mean separation is smaller than $\tcO(d^{1/4})$. 
\item \textbf{Recursive clustering}:  As long as there is a cluster
that has samples from more than one component with means 
\emph{far} apart, (described by a condition on 
the norm of its covariance matrix in the algorithm) estimate its
largest eigenvector and project samples of this cluster onto
this eigenvector and cluster them. This hierarchical method
is continued until there are clusters that contain close-by-components. 
\item \textbf{Search in the span:} The resulting clusters contain components that are close-by, \ie $\norm{\mupi -\mupj}_2< \cO(k^{3/2}{\wsigma}^2\log \frac{n}{\delta})$. We approximate the span of means by
the top $k-1$ eigenvectors and the mean vector, and perform an exhaustive search using \textsc{Modified scheffe}.
\end{enumerate}

We now describe these steps stating the performance 
of each step.


\begin{center}
\fbox{\begin{minipage}{1.0\textwidth}
Algorithm \textsc{Learn k-sphere} \newline
\textbf{Input:} $n$ samples $\xvec(1), \xvec(2), \ldots, \xvec(n)$ from $\fvec$ and $\epsilon$.
\begin{enumerate}
\item \textbf{Sample variance:} $\wqsigma^2 = \min_{a \neq b : a,b \in [k+1]}  \norm{\xvec(a)-\xvec(b)}^2_2/2d$.
\item \textbf{Coarse single-linkage clustering:} Start with each sample as a cluster,
\begin{itemize}
 \item While $\exists$ two clusters with squared-distance $\leq 2d\wqsigma^2 + 23 \wqsigma^2\sqrt{d\log\frac{n^2}{\delta}}$, merge them.
\end{itemize}
\item \textbf{Recursive spectral-clustering:}  While there is a new
  cluster $C$ with $|C| \geq n\epsilon/5k$ and spectral norm of its
sample covariance matrix
$\geq 12k^2 \wsigma^2 \log n^3/\delta$,
\begin{itemize}
\item Use $n\epsilon/8k^2$ of the samples to find the largest eigenvector and
  discard these samples.
\item Project the remaining samples on the largest eigenvector.
\item Perform single-linkage in the projected space (as before) till the distance between clusters $> 3\wqsigma \sqrt{\log n^2k/\delta}$
creating new clusters.
\end{itemize}
\item \textbf{Exhaustive search}: Let $\epsilon_g = \epsilon/(16k^{3/2})$, $L = 200  \sqrt{k^{4}\epsilon^{-1}\log \frac{n^2}{\delta}}$,
and $G = \{-L , \ldots,-\epsilon_g, 0,\epsilon_g,2\epsilon_g, \ldots L \}$. Let 
 $W = \{ 0,\epsilon/(4k), 2\epsilon/(4k), \ldots 1\}$
 and $\Sigma \ed \{\sigma^2: \sigma^2 = \wsigma^2 (1+i/d) \forall -d < i \leq d\}$.
 \begin{itemize}
\item For each cluster $C$ find its top $k-1$ eigenvectors $\uvec_1, \uvec_2\ldots
\uvec_{k-1}$ and let $\Span(C)=\{\wmuave(C)+\sum^{k-1}_{i=1} g_i\wsigma \uvec_i: g_1,g_2 \ldots g_{k-1} \in G\}$.
\item Let $\Span = \{\Span (C):|C| \geq n \epsilon/5k\}$.
\item 
For all $\wi' \in W$, $\sigma'^2 \in \Sigma$, $\wmupi \in \Span$, add $\{(w'_1, \ldots, w'_{k-1},1-\sum^{k-1}_{i=1} w'_i,N(\wmupone,\sigma'^2), \ldots, N(\wmupk,\sigma'^2) \}$ in $\cF$.
\end{itemize}
\item Run \textsc{modified scheffe} on $\cF$ and output the resulting distribution.
 \end{enumerate}
\end{minipage}}
\end{center} 
\subsection{Sketch of correctness}
To simplify the bounds and expressions, we assume that $d>1000$
and $\delta \geq  \min(2n^2 e^{-d/10},1/3)$. For smaller values of $\delta$, we run the algorithm 
with error $1/3$ and repeat it $\cO(\log \frac{1}{\delta})$ times to choose a set of candidate mixtures $\cF_{\delta}$.
By Chernoff-bound with error $\leq \delta$, $\cF_{\delta}$ contains a mixture $\epsilon$-close to $\fvec$.
Finally, we run \textsc{modified scheffe} on $\cF_\delta$ to obtain a mixture that is close to $\fvec$. By the union bound and Lemma~\ref{lem:scheffe}, the error is $\leq 2\delta$.

\textbf{Variance estimation:} Let $\wsigma$ be the variance estimate from step 1. 
In high dimensions, the difference between
two random samples from a Gaussian concentrates. This
is made precise in the next lemma which states
$\wsigma$ is a good estimate of the variance. 
Then the following is a simple application of Gaussian tail bounds.
\begin{Lemma}[Appendix~\ref{app:sigconc}]
\label{lem:sigconc}
Given $n$ samples from the $k$-component mixture,
 with probability $1-2\delta$, \[|\wqsigma^2 -\sigma^2| \leq 2.5 \sigma^2 \sqrt{\frac{\log (n^2/\delta)}{d}}.\]
\end{Lemma}

\textbf{Coarse single-linkage clustering:}
The second step is a single-linkage routine that clusters
mixture components with \emph{far} means.
Single-linkage is a simple clustering scheme that starts out with
each data point as a cluster, and at each step merges the two
that are closest to form larger clusters. The algorithm stops when
the distance between clusters is larger than a pre-specified threshold.

Suppose the samples are generated by an one-dimensional mixture
of $k$ components that are \emph{far}, then \emph{with high
probability}, when the 
algorithm generates $k$ clusters and all the samples within a cluster
are generated by a single component. 
More precisely, if $\forall i,j \in [k]$, $|\mu_i - \mu_j | = \Omega(\sigma \log n)$, 
then all the $n$ samples concentrate around their respective means
and the separation between any two samples from different components would be larger 
than the largest separation between any two samples from the same component. Hence for a suitable value of threshold, 
single-linkage correctly identifies the clusters.
For $d$-dimensional Gaussian mixtures
a similar notion holds true, with  minimum separation 
$\Omega(d^{1/4}\log \frac{n}{\delta})$. More precisely,
\begin{Lemma}[Appendix~\ref{app:singlelinkage}]
\label{lem:singlelinkage}
After Step 2 of \textsc{Learn k-sphere}, with probability $\geq 1-2\delta$, all
samples from each component will be in the same cluster
and the maximum distance between two components within each cluster is
$\leq 10k\sigma \bigl(d \log \frac{n^2}{\delta}\bigr)^{1/4}$.
\end{Lemma}

\textbf{Recursive spectral-clustering:}
The clusters formed at this step consists of components 
with mean separation $\cO(d^{1/4}\log \frac{n}{\delta})$. We now recursively
zoom into the clusters formed and show that it is possible
to cluster the components with much smaller mean separation. 
Note that since the matrix is symmetric, the largest  magnitude of the eigenvalue is same as the spectral norm.
We first find the largest eigenvector of 
\[
 S(C) \ed \frac{1}{|C|}\Big(\sum_{\xvec \in C} (\xvec - \wmuave(C))(\xvec - \wmuave(C))^t\Big) - \wsigma^2\II_d,
\]
which is the sample covariance matrix with its diagonal term reduced by
$\wsigma^2.$
If
there are two components with means far apart, then using single-linkage we divide the cluster into two. 
The following lemma
shows that this step performs accurate clustering
of components with means well separated. 


\begin{Lemma}[Appendix~\ref{app:reccluster}]
\label{lem:reccluster}
Let $n \geq c \cdot \frac{dk^4}{ \epsilon}\log \frac{n^3}{\delta}$. After recursive clustering, with probability $\geq 1-4\delta$.
the samples are divided into clusters such that for each component $i$
within any cluster $C$, $\sqrt{\wi}\norm{\mupi-\muave(C)}_2 \leq  25 \sigma \sqrt{k^3\log \frac{n^3}{\delta}}$
. Furthermore, all the samples from one component remain in a single cluster. 
\end{Lemma}

\textbf{Exhaustive search and Scheffe:} After step 3, all clusters have a small weighted radius
 $\sqrt{\wi}\norm{\mupi-\muave(C)}_2\le  25 \sigma \sqrt{k^3\log \frac{n^3}{\delta}}$, the the eigenvectors give an accurate estimate of the span of
$\mupi-\muave(C)$ within each cluster. More precisely,
\begin{Lemma}[Appendix~\ref{app:spancluster}]
\label{lem:spancluster}
Let  $n \geq c \cdot \frac{dk^9 }{\epsilon^{4}}\log^2 \frac{d}{\delta}$
for some constant $c$. After step 3, with probability $\geq 1-7\delta$ the following holds: if $|C| \geq n\epsilon/5k$,
 then the projection of $[\mupi -\muave(C)]/\norm{\mupi -\muave(C)}_2$
on the space orthogonal to the span of top $k-1$ eigenvectors has magnitude 
$\leq \frac{\epsilon \sigma}{8\sqrt{2}k \sqrt{\wi} \norm{\mupi -\muave(C)}}_2$.
\end{Lemma}

We now have accurate estimates of the spans of the clusters and each
cluster has components with close means. It is now possible to grid
the set of possibilities in each cluster to obtain a set of
distributions such that one of them is close
to the underlying. There is a trade-off between 
a dense grid to obtain a good estimation and the computation time
required. The final step takes the sparsest grid possible to ensure an error $\le
\epsilon$. This is quantized below.
\begin{Theorem}[Appendix~\ref{app:thmksphere}]
\label{thm:ksphere}
Let  $n \geq c \cdot \frac{dk^9 }{\epsilon^{4}} \log^2\frac d{\delta}$ for some constant $c$. Then Algorithm~\textsc{Learn k-sphere} with
probability $\geq 1- 9\delta$, outputs a distribution $\wfvec$ such
that $\lone{\wfvec}{\fvec}\le1000\epsilon$. Furthermore, the algorithm
runs in time
  $\cO\Big(n^2d \log n +
 d^2  \Big(\frac{k^{7}}{\epsilon^3}\log \frac{d}{\delta}\Big)^{k^2}  \Big) .$
\end{Theorem}
Note that the run time is calculated based on the efficient implementation of single-linkage~\cite{Sibson73} and the exponential term is not optimized. We now study mixtures in one-dimension and provide an estimator
using \textsc{Modified Scheffe}.

\section{Mixtures in one dimension}
\label{sec:one}

Over the past decade estimating one dimensional distributions has gained significant attention~\cite{Paninski04,Valiant10,DaskalakisDS12,DaskalakisDS12a,AcharyaJOS13,ChanDSS13,DaskalakisK13,ChanDSS13b}.
We now provide a simple estimator for learning one dimensional mixtures using the
\textsc{Modified Scheffe} estimator proposed earlier. 
The $d$-dimension estimator uses spectral projections to find the span of means, whereas for one dimension case,
we use a simple observation on properties of samples from Gaussians for estimation.
Formally, given samples from $f$, a mixture of Gaussian
distributions 
$p_i \ed N(\mu_i,\sigma^2_i)$ with weights $w_1,w_2, \ldots w_k$, 
our goal is to find a mixture $\hat{f} =(\wwone,\wwtwo, \ldots
\wwk,\hat{p}_1,\hat{p}_2,\ldots \hat{p}_k)$ 
such that $\lone{f}{\hat{f}} \leq \epsilon$. 
Note that we make no assumption on the weights, means
or the variances of the components. 

We provide an algorithm that, using $\tcO(k\epsilon^{-2})$ samples 
and in time $\tcO(k\epsilon^{-3k-1})$, outputs an estimate that is at
most $\epsilon$ from the underlying in $\ell_1$ distance with 
probability $\geq 1-\delta$.
Our algorithm is an immediate consequence of the following 
observation for samples from a Gaussian distribution.

\begin{Lemma}
\label{lem:goodsampleexists}
 Given $n$ independent samples $x_1, \ldots, x_n$  from $N(\mu,\sigma^2)$, 
 there are two samples $x_j,x_k$ such that $|x_j -\mu| \leq \sigma\frac{7\log 2/\delta}{2n}$ 
 and $|x_j-x_k-\sigma| \leq 2\sigma \frac{7\log 2/\delta}{2n}$
 with probability $\geq 1-\delta$.
\end{Lemma}
\begin{Proof}
The density of $N(\mu,\sigma^2)$ is $\geq (7\sigma)^{-1}$ in
the interval $[\mu-\sqrt{2}\sigma ,\mu + \sqrt{2}\sigma]$.
Therefore, the probability that a sample occurs in the interval
 $\mu-\epsilon\sigma,\mu + \epsilon\sigma$
 is $\geq 2\epsilon/7$. Hence, the probability that none of the 
$n$ samples occurs in $[\mu -\epsilon\sigma,\mu + \epsilon\sigma]$
 is $\leq(1-2\epsilon/7)^n \leq e^{-2n\epsilon/7}$. If $\epsilon \geq \frac{7\log 2/\delta}{2n}$, then the probability that
 none of the samples occur in the interval is $\leq \delta/2$. 
 A similar argument shows that there is a sample within interval,
$[\mu+\sigma-\epsilon\sigma,\mu + \sigma + \epsilon\sigma]$, proving the lemma.
\end{Proof}

The above observation can be translated into selecting a pool of candidate distributions
such that one of the distributions is close to the underlying distribution.

\begin{Lemma}
\label{lem:coolsearch}
Given $n \geq \frac{120k \log \frac{4k}{\delta}}{\epsilon}$ samples from 
a mixture $f$ of $k$ Gaussians. 
Let $S = \{N(x_j,(x_j-x_k)^2) \, : \,1\leq j,k \le n\}$ be a set of Gaussians
and $W = \{0, \frac{\epsilon}{2k} ,\frac{2\epsilon}{2k} \ldots ,1\}$ be the set of weights.
Let 
\[\cF \ed \{\wwone,\wwtwo, \ldots,\hat{w}_{k-1}, 1 -\sum^{k-1}_{i=1}\wwi,\hat{p}_1,\hat{p}_2,\ldots \hat{p}_k  \, : \,  \wwi \in W , \hat{p}_i \in S\}\]
be a set of $n^{2k} \bigl( \frac{2k}{\epsilon} \bigr)^{k-1} \leq n^{3k-1}$ candidate mixture distributions.
There exists a $\hat{f} \in \cF$ such that $\lone{f}{\hat{f}} \leq \epsilon$.
\end{Lemma}
\begin{proof}
Let
$f = (\wone,\wtwo, \ldots w_k, p_1,p_2,\ldots p_k)$. For 
$\hat{f}  = (\wwone,\wwtwo, \ldots,\hat{w}_{k-1}, 1 -\sum^{k-1}_{i=1}\wwi,\hat{p}_1,\hat{p}_2,\ldots \hat{p}_k )$, 
by the  triangle
inequality,
\begin{align*}
\lone{f}{\hat{f}} \leq \sum^{k-1}_{i=1} 2|\wwi - \wi|  + \sum^k_{i=1} \wi \lone{p_i}{\hat{p}_i}.
\end{align*}
We show that there is a distribution in $\hat{f}\in\cF$ such that the
sum above is bounded by $\epsilon$. Since we quantize the grids as
multiples of $\epsilon/2k$, we consider distributions in $\cF$ such 
that each $|\wwi - \wi| \leq \epsilon/4k$, and therefore $\sum_i|\wwi -
\wi| \leq \frac{\epsilon}{2}$.

We now show that for each $p_i$ there is a $\hat{p}_i$ such that $\wi \lone{p_i}{\hat{p}_i} \leq \frac{\epsilon}{2k}$,
thus proving that $\lone{f}{\hat{f}} \leq \epsilon$.
If $\wi\leq \frac{\epsilon}{4k}$, then $\wi \lone{p_i}{\hat{p}_i} \leq \frac{\epsilon}{2k}$.
Otherwise, let $w'_i>\frac{\epsilon}{4k}$ be the fraction of samples
from $p_i$. By Lemma~\ref{lem:goodsampleexists} and~\ref{lem:l1bha},
with probability $\geq 1- \delta/2k$,
\begin{align*}
 \lone{p_i}{\hat{p}_i}^2  &\leq 2\frac{(\mu_i- \mu'_i)^2}{\sigma^2_i} + 16 \frac{(\sigma_i - \sigma'_i)^2}{\sigma^2_i} \\
 & \leq \frac{25 \log^2 \frac{4k}{\delta}}{(nw'_i)^2} + \frac{800 \log^2 \frac{4k}{\delta}}{(nw'_i)^2}\\
 & \leq \frac{825 \log^2 \frac{4k}{\delta}}{(nw'_i)^2}.
\end{align*}
Therefore,
\[
\wi \lone{p_i}{\hat{p}_i} \leq \frac{30 \wi \log \frac{4k}{\delta}}{nw'_i}.
\]
Since $\wi > \epsilon/4k$, with probability $\geq 1- \delta/2k$, $\wi \leq 2 w'_i$.
By the union bound with probability $\geq 1- \delta/k$, 
$\wi \lone{p_i}{\hat{p}_i} \leq \frac{60 \log \frac{4k}{\delta}}{n}$.
Hence if $n \geq \frac{120 k \log \frac{4k}{\delta}}{\epsilon}$, the 
above quantity is less than $\epsilon/2k$. The total error probability is $\leq \delta$
by the union bound.
\end{proof}
Running \textsc{Modified Scheffe} algorithm on the above set of
candidates $\cF$ 
yields a mixture that is close to the underlying one.
By Lemma~\ref{lem:scheffe} and the above lemma we get
\begin{Corollary}
\label{cor:onedimension}
Let $n \geq   c \cdot \frac{k\log \frac{k}{\epsilon\delta}}{\epsilon^2} $ for some constant $c$.
There is an algorithm that
runs in time 
\[
\cO\left( \left( \frac{k \log \frac{k}{\epsilon \delta}}{\epsilon} \right)^{3k-1} \frac{k^2 \log \frac{k}{\epsilon \delta}}{\epsilon^2} \right),
\]
 and returns a mixture $\hat{f}$ such that $\lone{f}{\hat{f}} \leq 1000\epsilon$ with error probability $\leq 2 \delta$.
\end{Corollary}
\begin{proof}
 Use $n' \ed \frac{120k \log \frac{4k}{\delta}}{\epsilon}$ samples to 
 generate a set of at most $n'^{3k-1}$ candidate distributions as stated in Lemma~\ref{lem:coolsearch}.
 With probability $\geq 1- \delta$, one of the candidate distributions is $\epsilon$-close to the underlying one.
 Run \textsc{Modified Scheffe} on this set of candidate distributions to obtain a $1000\epsilon$-close estimate of $f$ with probability $\geq 1-\delta$ (Lemma~\ref{lem:scheffe}).
 The run time is dominated by the run time of \textsc{Modified Scheffe} which is $\cO \left(\frac{ |\cF| T \log \frac{|\cF|}{\delta}}{\epsilon^2} \right)$,
 where $|\cF| = n'^{3k-1}$ and $T = k$. The total error probability is $\leq 2 \delta$ by the union bound.
\end{proof}
\begin{Remark}
The above bound matches the independent and contemporary 
result by~\cite{DaskalakisK13} for $k=2$. While the process of identifying the candidate means is same for both the papers, the process of identifying
the variances and proof techniques are different.
\end{Remark}
\section{Acknowledgements}
We thank Sanjoy Dasgupta, Todd Kemp, and Krishnamurthy Vishwanathan for helpful discussions.
\bibliographystyle{plain}
\bibliography{abr,masterref}

\begin{thebibliography}{10}

\bibitem{AcharyaJOS13}
Jayadev Acharya, Ashkan Jafarpour, Alon Orlitsky, and Ananda~Theertha Suresh.
\newblock Optimal probability estimation with applications to prediction and
  classification.
\newblock In {\em Proceedings of the 26th Annual Conference on Learning Theory
  (COLT)}, pages 764--796, 2013.

\bibitem{AchlioptasM05}
Dimitris Achlioptas and Frank McSherry.
\newblock On spectral learning of mixtures of distributions.
\newblock In {\em Proceedings of the 18th Annual Conference on Learning Theory
  (COLT)}, pages 458--469, 2005.

\bibitem{AhlswedeW02}
Rudolf Ahlswede and Andreas Winter.
\newblock Strong converse for identification via quantum channels.
\newblock {\em IEEE Transactions on Information Theory}, 48(3):569--579, 2002.

\bibitem{AndersonBGRV13}
Joseph Anderson, Mikhail Belkin, Navin Goyal, Luis Rademacher, and James~R.
  Voss.
\newblock The more, the merrier: the blessing of dimensionality for learning
  large gaussian mixtures.
\newblock {\em CoRR}, abs/1311.2891, 2013.

\bibitem{AzizyanSW13}
Martin Azizyan, Aarti Singh, and Larry~A. Wasserman.
\newblock Minimax theory for high-dimensional gaussian mixtures with sparse
  mean separation.
\newblock {\em CoRR}, abs/1306.2035, 2013.

\bibitem{BelkinS10}
Mikhail Belkin and Kaushik Sinha.
\newblock Polynomial learning of distribution families.
\newblock In {\em Proceedings of the 51st Annual Symposium on Foundations of
  Computer Science (FOCS)}, pages 103--112, 2010.

\bibitem{ChaudhuriDV09}
Kamalika Chaudhuri, Sanjoy Dasgupta, and Andrea Vattani.
\newblock Learning mixtures of gaussians using the k-means algorithm.
\newblock {\em CoRR}, abs/0912.0086, 2009.

\bibitem{ColemanGC79}
G.B. Coleman and Harry~C. Andrews.
\newblock Image segmentation by clustering.
\newblock {\em Proceedings of the IEEE}, 67(5):773--785, 1979.

\bibitem{CoverT06}
Thomas~M. Cover and Joy~A. Thomas.
\newblock {\em Elements of information theory (2. ed.)}.
\newblock Wiley, 2006.

\bibitem{Dasgupta99}
Sanjoy Dasgupta.
\newblock Learning mixtures of gaussians.
\newblock In {\em Proceedings of the 40th Annual Symposium on Foundations of
  Computer Science (FOCS)}, pages 634--644, 1999.

\bibitem{DasguptaS00}
Sanjoy Dasgupta and Leonard~J. Schulman.
\newblock A two-round variant of {EM} for gaussian mixtures.
\newblock In {\em Proceedings of the 16th Annual Conference on Uncertainty in
  Artificial Intelligence (UAI)}, pages 152--159, 2000.

\bibitem{DasguptaS07}
Sanjoy Dasgupta and Leonard~J. Schulman.
\newblock A probabilistic analysis of {EM} for mixtures of separated, spherical
  gaussians.
\newblock {\em Journal on Machine Learning Research (JMLR)}, 8:203--226, 2007.

\bibitem{DaskalakisDS12a}
Constantinos Daskalakis, Ilias Diakonikolas, and Rocco~A. Servedio.
\newblock Learning {\it k}-modal distributions via testing.
\newblock In {\em SODA}, pages 1371--1385, 2012.

\bibitem{DaskalakisDS12}
Constantinos Daskalakis, Ilias Diakonikolas, and Rocco~A. Servedio.
\newblock Learning poisson binomial distributions.
\newblock In {\em Proceedings of the 44th Annual Annual ACM Symposium on Theory
  of Computing (STOC)}, pages 709--728, 2012.

\bibitem{DaskalakisK13}
Constantinos Daskalakis and Gautam Kamath.
\newblock Faster and sample near-optimal algorithms for proper learning
  mixtures of gaussians.
\newblock {\em CoRR}, abs/1312.1054, 2013.

\bibitem{DevroyeL01}
Luc Devroye and G{\'a}bor Lugosi.
\newblock {\em Combinatorial methods in density estimation}.
\newblock Springer, 2001.

\bibitem{DhillonGK02}
Inderjit~S. Dhillon, Yuqiang Guan, and Jacob Kogan.
\newblock Iterative clustering of high dimensional text data augmented by local
  search.
\newblock In {\em Proceedings of the 2nd Industrial Conference on Data Mining
  (ICDM)}, pages 131--138, 2002.

\bibitem{FeldmanOS05}
Jon Feldman, Ryan O'Donnell, and Rocco~A. Servedio.
\newblock Learning mixtures of product distributions over discrete domains.
\newblock In {\em Proceedings of the 46th Annual Symposium on Foundations of
  Computer Science (FOCS)}, pages 501--510, 2005.

\bibitem{FeldmanSO06}
Jon Feldman, Rocco~A. Servedio, and Ryan O'Donnell.
\newblock {PAC} learning axis-aligned mixtures of gaussians with no separation
  assumption.
\newblock In {\em Proceedings of the 19th Annual Conference on Learning Theory
  (COLT)}, pages 20--34, 2006.

\bibitem{FreundM99}
Yoav Freund and Yishay Mansour.
\newblock Estimating a mixture of two product distributions.
\newblock In {\em Proceedings of the 13th Annual Conference on Learning Theory
  (COLT)}, pages 53--62, 1999.

\bibitem{HsuK13}
Daniel Hsu and Sham~M. Kakade.
\newblock Learning mixtures of spherical gaussians: moment methods and spectral
  decompositions.
\newblock In {\em Proceedings of the 4th Innovations in Theoretical Computer
  Science Conference (ITCS)}, pages 11--20, 2013.

\bibitem{KalaiMV10}
Adam~Tauman Kalai, Ankur Moitra, and Gregory Valiant.
\newblock Efficiently learning mixtures of two gaussians.
\newblock In {\em Proceedings of the 42nd Annual Annual ACM Symposium on Theory
  of Computing (STOC)}, pages 553--562, 2010.

\bibitem{KannanSV08}
Ravindran Kannan, Hadi Salmasian, and Santosh Vempala.
\newblock The spectral method for general mixture models.
\newblock {\em SIAM Journal on Computing}, 38(3):1141--1156, 2008.

\bibitem{KearnsMRRSS94}
Michael~J. Kearns, Yishay Mansour, Dana Ron, Ronitt Rubinfeld, Robert~E.
  Schapire, and Linda Sellie.
\newblock On the learnability of discrete distributions.
\newblock In {\em Proceedings of the 26th Annual Annual ACM Symposium on Theory
  of Computing (STOC)}, pages 273--282, 1994.

\bibitem{LaurentM00}
B.~Laurent and Pascal Massart.
\newblock Adaptive estimation of a quadratic functional by model selection.
\newblock {\em The Annals of Statistics}, 28(5):pp. 1302--1338, 2000.

\bibitem{Lindsay95}
Bruce~G. Lindsay.
\newblock {\em Mixture Models: Theory, Geometry and Applications}.
\newblock NSF-CBMS Conference series in Probability and Statistics, Penn. State
  University, 1995.

\bibitem{MaXJ01}
Jinwen Ma, Lei Xu, and Michael~I. Jordan.
\newblock Asymptotic convergence rate of the em algorithm for gaussian
  mixtures.
\newblock {\em Neural Computation}, 12(12):2881--2907, 2001.

\bibitem{MahalanabisS08}
Satyaki Mahalanabis and Daniel Stefankovic.
\newblock In {\em Proceedings of the 21st Annual Conference on Learning Theory
  (COLT)}, pages 503--512. Omnipress, 2008.

\bibitem{MoitraV10}
Ankur Moitra and Gregory Valiant.
\newblock Settling the polynomial learnability of mixtures of gaussians.
\newblock In {\em Proceedings of the 51st Annual Symposium on Foundations of
  Computer Science (FOCS)}, pages 93--102, 2010.

\bibitem{ChanDSS13b}
Siu on~Chan, Ilias Diakonikolas, Rocco~A. Servedio, and Xiaorui Sun.
\newblock Efficient density estimation via piecewise polynomial approximation.
\newblock {\em CoRR}, abs/1305.3207, 2013.

\bibitem{ChanDSS13}
Siu on~Chan, Ilias Diakonikolas, Rocco~A. Servedio, and Xiaorui Sun.
\newblock Learning mixtures of structured distributions over discrete domains.
\newblock In {\em Proceedings of the 24th Annual Symposium on Discrete
  Algorithms (SODA)}, pages 1380--1394, 2013.

\bibitem{OrlitskySVZ04}
Alon Orlitsky, Narayana~P. Santhanam, Krishnamurthy Viswanathan, and Junan
  Zhang.
\newblock On modeling profiles instead of values.
\newblock In {\em Proceedings of the 20th Annual Conference on Uncertainty in
  Artificial Intelligence (UAI)}, 2004.

\bibitem{Paninski04}
Liam Paninski.
\newblock Variational minimax estimation of discrete distributions under kl
  loss.
\newblock In {\em Proceedings of the 18th Annual Conference on Neural
  Information Processing (NIPS)}, 2004.

\bibitem{Pollard13}
David Pollard.
\newblock {\em Asymptopia}.
\newblock 1997.

\bibitem{RednerW84}
Richard~A. Redner and Homer~F. Walker.
\newblock Mixture densities, maximum likelihood and the em algorithm.
\newblock {\em SIAM Review}, 26(2):pp. 195--239, 1984.

\bibitem{ReynoldsR95}
Douglas~A. Reynolds and Richard~C. Rose.
\newblock Robust text-independent speaker identification using gaussian mixture
  speaker models.
\newblock {\em IEEE Transactions on Speech and Audio Processing}, 3(1):72--83,
  1995.

\bibitem{Sibson73}
Robin Sibson.
\newblock Slink: An optimally efficient algorithm for the single-link cluster
  method.
\newblock {\em The Computer Journal}, 16(1):30--34, 1973.

\bibitem{TitteringtonSM85}
D~Michael Titterington, Adrian~FM Smith, and Udi~E Makov.
\newblock {\em Statistical analysis of finite mixture distributions}, volume~7.
\newblock Wiley New York, 1985.

\bibitem{Tropp12}
Joel~A. Tropp.
\newblock User-friendly tail bounds for sums of random matrices.
\newblock {\em Foundations of Computational Mathematics}, 12(4):389--434, 2012.

\bibitem{Valiant11b}
G.~Valiant and P.~Valiant.
\newblock Estimating the unseen: an n/log(n)-sample estimator for entropy and
  support size, shown optimal via new clts.
\newblock Proceedings of the 43rd Annual Annual ACM Symposium on Theory of
  Computing (STOC), 2011.

\bibitem{Valiant10}
Gregory Valiant and Paul Valiant.
\newblock Estimating the unseen: A sublinear-sample canonical estimator of
  distributions.
\newblock {\em Electronic Colloquium on Computational Complexity (ECCC)},
  17:180, 2010.

\bibitem{VempalaW02}
Santosh Vempala and Grant Wang.
\newblock A spectral algorithm for learning mixtures of distributions.
\newblock In {\em Proceedings of the 43rd Annual Symposium on Foundations of
  Computer Science (FOCS)}, pages 113--, 2002.

\bibitem{Vershynin10}
Roman Vershynin.
\newblock Introduction to the non-asymptotic analysis of random matrices.
\newblock {\em CoRR}, abs/1011.3027, 2010.

\bibitem{XingJK01}
Eric~P. Xing, Michael~I. Jordan, and Richard~M. Karp.
\newblock Feature selection for high-dimensional genomic microarray data.
\newblock In {\em Proceedings of the 18th Annual International Conference on
  Machine Learning (ICML)}, pages 601--608, 2001.

\bibitem{Yu97}
Bin Yu.
\newblock {A}ssouad, {F}ano, and {L}e {C}am.
\newblock In {\em Festschrift for Lucien Le Cam}, pages 423--435. Springer New
  York, 1997.

\end{thebibliography}
 \appendix
%
\section{Useful tools}
\label{app:tools}
\subsection{Bounds on $\ell_1$ distance}
For two $d$ dimensional product distributions $\pvecone$ and $\pvectwo$, if we bound the $\ell_1$ distance on each coordinate by $\epsilon$,
then by triangle inequality $\lone{\pvecone}{\pvectwo}\leq d \epsilon$. However this bound is often weak.
One way to obtain a stronger bound is to relate $\ell_1$ distance to Bhattacharyya parameter, which is defined as follows:
Bhattacharyya parameter $B(p_1,p_2) $ between two distributions $p_1$ and $p_2$ is 
\[
B(p_1,p_2) = \int_{x \in \cX} \sqrt{p_1(x) p_2(x)} dx.
\]
We use the fact that for two product distributions $\pvecone$ and $\pvectwo$, $B(\pvecone,\pvectwo) = \prod^d_{i=1} B(\pveconei,\pvectwoi)$
to obtain stronger bounds on the $\ell_1$ distance. We first bound Bhattacharyya parameter for two one-dimensional Gaussian distributions.
\begin{Lemma}
\label{lem:bhagpb}
The Bhattacharyya parameter for two one dimensional Gaussian distributions $p_1 = N(\mu_1,\sigma_1^2)$ and $p_2 =N(\mu_2,\sigma_2^2)$ is
\[
B(p_1,p_2) \geq 1 - \frac{(\mu_1-\mu_2)^2)}{4(\sigma_1^2 + \sigma_2^2)}  -   \frac{(\sigma_1^2-\sigma_2^2)^2}{(\sigma_1^2+\sigma_2^2)^2}.
\]
\end{Lemma}
\begin{proof}
For Gaussian distributions the Bhattacharyya parameter is (see~\cite{ColemanGC79}), $B(p_1,p_2) = ye^{-x}$,
where $x = \frac{(\mu_1-\mu_2)^2)}{4(\sigma_1^2 + \sigma_2^2)}$ and $y  = \sqrt{\frac{2\sigma_1 \sigma_2}{\sigma_1^2  +\sigma_2^2}}$ .
Observe that 
\[
y = \sqrt{\frac{2\sigma_1 \sigma_2}{\sigma_1^2  +\sigma_2^2}} = \sqrt{1 - \frac{(\sigma_1-\sigma_2)^2}{\sigma_1^2+\sigma_2^2}}
\geq 1 - \frac{(\sigma_1-\sigma_2)^2}{\sigma_1^2+\sigma_2^2} \geq 1 -  \frac{(\sigma_1^2-\sigma_2^2)^2}{(\sigma_1^2+\sigma_2^2)^2}.
\]
Hence,
\[
B(p_1,p_2) = y e^{-x} \geq y(1-x) \geq (1-x)\biggl(  1 -  \frac{(\sigma_1^2-\sigma_2^2)^2}{(\sigma_1^2+\sigma_2^2)^2} \biggr)  \geq 1-x-\frac{(\sigma_1^2-\sigma_2^2)^2}{(\sigma_1^2+\sigma_2^2)^2} .
\]
Substituting the value of $x$ results in the lemma.
\end{proof}
 The next lemma follows from the relationship between Bhattacharyya parameter and $\ell_1$ distance (see~\cite{Pollard13}),
 and the previous lemma.
\begin{Lemma}
\label{lem:l1bha}
For any two Gaussian product distributions $\pvecone$ and $\pvectwo$,
\[
\lone{\pvecone}{\pvectwo}^2 \leq 8 \biggl( \sum^{d}_{i=1} 1 - B(\pveconei,\pvectwoi) \biggr) \leq  \sum^d_{i=1} 2\frac{(\muponei-\muptwoi)^2}{\sigponei^2 + \sigptwoi^2}  + 8\frac{(\sigponei^2-\sigptwoi^2)^2}{(\sigponei^2+\sigptwoi^2)^2}.
\]
\end{Lemma}

\subsection{Concentration inequalities}
We use the following concentration inequalities for Gaussian, Chi-Square, and sum of Bernoulli random variables in the rest of the paper.
\begin{Lemma}
 \label{lem:gaussbound}
 For a Gaussian random variable $X$ with mean $\mu$ and variance $\sigma^2$,
 \[
  \Pr (|X-\mu| \geq t \sigma ) \leq e^{-t^2/2}.
 \]
\end{Lemma}
\begin{Lemma}[\cite{LaurentM00}]
\label{lem:tailchi}
If $Y_1, Y_2, \ldots Y_n$ be $n$ \iid Gaussian variables with mean $0$ and variance $\sigma^2$, then
 \[
  \Pr\biggl( \sum^n_{i=1} Y^2_i - n \sigma^2 \geq 2 (\sqrt{nt} + t)\sigma^2 \biggr) \leq  e^{-t}, \text{ and }
  \Pr\biggl( \sum^n_{i=1} Y^2_i - n \sigma^2 \leq  -2\sqrt{nt} \sigma^2 \biggr) \leq  e^{-t}.
  \]
  Furthermore for a fixed vector $\avec$,
   \[
  \Pr\biggl( \lV \sum^n_{i=1} \avec_i( Y^2_i -1) \rV \leq  2(\norm{\avec}_2\sqrt{t} + \norm{\avec}_\infty t)\sigma^2 \biggr) \leq  2e^{-t}.
  \]
\end{Lemma}
\begin{Lemma}[Chernoff bound]
\label{lem:eqchern}
If $X_1, X_2 \ldots X_n$ are distributed according to Bernoulli $p$, then with probability $1 - \delta$,
\[
 \lV \frac{\sum^n_{i=1}X_i}{n} - p \rV \leq \sqrt{ \frac{2p(1-p)}{n} \log \frac{2}{\delta}} + \frac{2}{3} \frac{\log \frac{2}{\delta}}{n}.
\]
\end{Lemma}
We now state a non-asymptotic concentration inequality for random matrices that helps us bound errors in spectral algorithms.
\begin{Lemma}[\cite{Vershynin10} Remark $5.51$]
\label{lem:matconc}
Let $\yvec(1), \yvec(2),\ldots, \yvec(n)$ be generated according to $N(0,\Sigma)$. For every $\epsilon \in (0,1)$ 
and $t \geq 1$, if $n \geq c' d \bigl(\frac{t}{\epsilon}\bigr)^2$ for some constant $c'$, then with probability $\geq 1 -2 e^{-t^2n}$,
\[
\norm{\sum^n_{i=1} \frac{1}{n} \yvec(i) \yvec^t(i) - \Sigma} \leq \epsilon \norm{\Sigma}.
\]
\end{Lemma}

\subsection{Matrix eigenvalues}
We now state few simple lemmas on the eigenvalues of perturbed matrices.
\begin{Lemma}
\label{lem:closeeigen}
Let $\lambda^A_1 \geq \lambda^A \geq \ldots \lambda^A_d \geq 0$ and  $\lambda^B_1 \geq \lambda^B \geq \ldots \lambda^B_d \geq 0$ be the eigenvalues of 
two symmetric matrices $A$ and $B$ respectively. If $\norm{A-B} \leq \epsilon$, then $\forall \,i$, $|\lambda^A_i  - \lambda^B_i| \leq \epsilon$.
\end{Lemma}
\begin{proof}
 Let $\uvec_1,\uvec_2, \ldots \uvec_d$ be a set of eigenvectors of $A$ that corresponds to
 $\lambda^A_1, \lambda^A_2 , \ldots \lambda^A_d$. Similarly let  $\vvec_1,\vvec_2, \ldots \vvec_d$ be eigenvectors of $B$ 
 Consider the first eigenvalue of $B$,
 \[
\lambda^B_1 = \norm{B} = \norm{A+(B-A)} \geq \norm{A} - \norm{B-A} \geq \lambda^A_1 - \epsilon.  
 \]
Now consider an $i > 1$.  If $\lambda^B_i < \lambda^A_i - \epsilon$, 
 then by definition of eigenvalues
 \[
  \max_{\vvec : \forall j \leq i-1,  \vvec \cdot \vvec_j =0} \norm{B\vvec}_2 < \lambda^A_i - \epsilon.
 \]
Now consider a unit vector $\sum^i_{j=1} \alpha_j \uvec_j$ in the span of $\uvec_1, \ldots \uvec_{i}$, that is orthogonal to $\vvec_1, \ldots \vvec_{i-1}$.
For this vector,
\[
 \norm{B\sum^i_{j=1} \alpha_j \uvec_j}_2 \geq \norm{A \sum^i_{j=1} \alpha_j \uvec_j}_2 - \norm{(A-B) \sum^i_{j=1} \alpha_j \uvec_j}_2 \geq \sqrt{\sum^i_{j=1} \alpha^2_j (\lambda^A_j)^2 } - \epsilon
 \geq \lambda^A_i - \epsilon,
\]
a contradiction. Hence, $\forall i \leq d$, $\lambda^B_i  \geq \lambda^A_i - \epsilon $. The proof in the other direction is similar and omitted.
\end{proof}

\begin{Lemma}
\label{lem:matspan}
 Let $A = \sum^k_{i=1}\eta^2_i \uvec_i \uvec^t_i$ be a positive semidefinite symmetric matrix for $k \leq d$. Let $\uvec_1,\uvec_2, \ldots \uvec_{k}$ span a $k-1$ dimensional space.
 Let $B = A+R$, where $\norm{R} \leq \epsilon$.
Let $\vvec_1 , \vvec_2, \ldots \vvec_{k-1}$ be the top $k-1$ eigenvectors of $B$. Then the projection of $\uvec_i$ in space orthogonal to
$\vvec_1 , \vvec_2, \ldots \vvec_{k-1}$
is $\leq  \frac{2\sqrt{\epsilon}}{\eta_i}$.
\end{Lemma}
\begin{proof}
Let $\lambda^B_i$ be the $i^{th}$ largest eigenvalue of $B$. Observe that $B+\epsilon \II_d$ is a positive semidefinite matrix as for any vector $\vvec$, $\vvec^t(A+R+\epsilon\II_d) \vvec \geq 0$.
Furthermore $\norm{A+R+\epsilon\II_d - A} \leq 2 \epsilon$. Since eigenvalues of $B+\epsilon \II_d$ is $\lambda^B+\epsilon$, by Lemma~\ref{lem:closeeigen}, 
for all $i \leq d$, $|\lambda^A_i - \lambda^B_i - \epsilon|\leq 2 \epsilon$. Therefore, $|\lambda^B_{i}|$ for $i \geq k$ is $\leq 3\epsilon$.

Let $\uvec_i = \sum^{k-1}_{j=1} \alpha_{i,j} \vvec_j + \sqrt{ 1-  \sum^{k-1}_{j=1} \alpha^2_{i,j}} \uvec'$,
for a vector $\uvec'$ orthogonal to $\vvec_1 , \vvec_2, \ldots \vvec_{k-1}$.
We compute $\uvec'^t A \uvec'$ in two ways. Since $A = B-R$,
\[
|\uvec'^t(B-R) \uvec'| \leq |\uvec'^tB \uvec'| + |\uvec'^tR \uvec' | \leq \norm{B \uvec'}_2 + \norm{R}.
\]
 Since $\uvec'$ is orthogonal to first $k$ eigenvectors,
we have 
$\norm{B \uvec'}_2 \leq 3 \epsilon$ and hence $|\uvec'^(B-R) \uvec'| \leq 4 \epsilon$.
\[
\uvec'^t A \uvec' \geq\eta^2_i \bigl(  1-  \sum^{k-1}_{j=1} \alpha^2_{i,j} \bigr).
\]
We have shown that the above quantity is $\leq 4 \epsilon$.
Therefore $\bigl(  1-  \sum^{k-1}_{j=1} \alpha^2_{i,j} \bigr)^{1/2}\leq 2 \sqrt{\epsilon}/\eta_i$.
\end{proof}

\section{Selection from a set of candidate distributions}
\label{app:Scheffe}


%

%
%
Given samples from an unknown distribution $f$, the objective is to output
a distribution from a known collection $\cF$ of distributions with $\ell_1$
distance close to $\lone{f}{\cF}$. Scheffe estimate~\cite{DevroyeL01}
outputs a distribution from $\cF$
whose $\ell_1$ distance from $f$ is at most $9.1\max(
\lone{f}{\cF}, \epsilon)$ The algorithm
requires $\cO(\epsilon^{-2}\log |\cF|)$ samples and the runs in time
$\cO(|\cF|^2T(n+|\cX|))$,
where $T$ is the time to compute the probability $f_j(x)$ of $x$, 
for any $f_j\in\cF$. 
An approach to reduce the time complexity, albeit using
exponential pre-processing, was
proposed in~\cite{MahalanabisS08}. We present the
modified Scheffe algorithm with near linear time
complexity and then prove Lemma~\ref{lem:scheffe}.

We first present the algorithm \textsc{Scheffe*} with running time
 $\tcO(|\cF|^2Tn)$.
\begin{center}
\fbox{\begin{minipage}{1.0\textwidth}
Algorithm \textsc{ Scheffe*} \newline
\textbf{Input: }a set $\cF$ of candidate distributions, $\epsilon:$ upper 
bound on $\lone{f}{\cF}$, 
$n$ independent 
samples $x_1, \ldots, x_{n}$ from $f$.\\
\newline
For each pair $(p,q)$ in $\cF$ do:
\begin{enumerate}
\item
$\emp_f=\frac 1{n} \sum^n_{i=1}\mathbb{I}\{p(x_i)>q.(x_i)\}$.
\item
Generate independent samples $y_1, \ldots, y_{n}$ and 
$z_1, \ldots, z_{n}$ from $p$ and $q$ respectively.
\item
$\emp_p= \frac1{n} \sum^{n}_{i=1}\mathbb{I}\{p(y_i)>q(y_i)\}$,
 $\emp_q= \frac1{n} \sum^{n}_{i=1} \mathbb{I}\{p(z_i)>q(z_i)\}$.
\item
If $|\emp_p-\emp_f|<|\emp_q-\emp_f|$ declare $p$ as winner, else $q$.
\end{enumerate}
Output the distribution with most wins, breaking ties arbitrarily.
\end{minipage}}
\end{center}

We make the following modification to the algorithm where
we reduce the size of potential distributions by half in 
every iteration. 

\begin{center}
\fbox{\begin{minipage}{1.0\textwidth}
Algorithm \textsc{modified Scheffe} \newline
\textbf{Input: } set $\cF$ of candidate distributions, 
$\epsilon:$ upper bound on $\min_{f_i \in \cF} \lone{f}{f_i}$,
$n$ independent samples $x_1,
\ldots, x_{n}$ from $f$.
\begin{enumerate}
\item
Let $\cG=\cF$, $\cC\leftarrow\emptyset$
\item
Repeat until $|\cG|>1$:
\begin{enumerate}
\item
Randomly form $|\cG|/2$ pairs of distributions in $\cG$ 
and run \textsc{Scheffe*} on \emph{each pair} using the 
$n$ samples.
\item
Replace $\cG$ with the  $|\cG|/2$ winners.
\item
Randomly select a set $\cA$ of  $\min\{|\cG|, |\cF|^{1/3}\}$ elements from $\cG$.
\item
Run \textsc{Scheffe*} on each pair in $\cA$ and add the distributions
with most wins to $\cC$.
\end{enumerate}
\item
Run \textsc{Scheffe*} on $\cC$ and output the winner
\end{enumerate}
\end{minipage}}
\end{center}

\begin{Remark}
 For the ease of proof, we assume that $\delta \geq \frac{10\log |\cF|}{|\cF|^{1/3}}$. If $\delta < \frac{10\log |\cF|}{|\cF|^{1/3}}$, we run the algorithm 
with error probability $1/3$ and repeat it $\cO(\log \frac{1}{\delta})$ times to choose a set of candidate mixtures $\cF_{\delta}$.
By Chernoff-bound with error probability $\leq \delta$, $\cF_{\delta}$ contains a mixture close to $f$.
Finally, we run \textsc{Scheffe*} on $\cF_\delta$ to obtain a mixture that is close to $f$.
\end{Remark}

\begin{proof}[Proof sketch of Lemma~\ref{lem:scheffe}]

For any set $\cA$ and a distribution $p$, given $n$
independent samples from $p$ the empirical probability
$\emp_{n}(\cA)$ has a distribution around $p(\cA)$ 
with standard deviation $\sim\frac1{\sqrt{n}}$.
Together with an observation in Scheffe estimation in~\cite{DevroyeL01}
one can show that if the number of samples $n = \cO \left(\frac{\log \frac{|\cF|}{\delta}}{\epsilon^2} \right)$,
then \textsc{Scheffe*} has a guarantee $10 \max(\epsilon, \lone{f}{\cF})$ with probability $\geq 1-\delta$. 

Since we run \textsc{Scheffe*} at most $|\cF|(2\log |\cF|+1)$ times,
choosing $\delta = \delta/(4|\cF|\log |\cF|+2|\cF|)$ results in the sample complexity of
\[\cO \left(\frac{\log \frac{|\cF|^2(4\log |\cF|+2)}{\delta}}{\epsilon^2} \right) = \cO\left(\frac{\log \frac{|\cF|}{\delta}}{\epsilon^2}\right),\]
and the total error probability of $\delta/2$ for all runs of 
\textsc{Scheffe*} during the algorithm.
The above value of $n$ dictates our sample complexity. We now consider the following two cases:
\begin{itemize}
\item
If at some stage $\ge \frac{\log  (2/\delta)}{|\cF|^{1/3}}$
 fraction of elements in $\cA$ have an $\ell_1$ distance
$\le 10\epsilon$ from $f$, then at that stage with probability 
$\ge1-\delta/2$ an element with distance $\le10\epsilon$ from $f$
is added to $\cA$.
Therefore a distribution with distance $\le 100\epsilon$
is selected to $\cC$.
\item
If at no stage this happens, then consider the element that 
is closest to $f$, \ie at $\ell_1$ distance at most $\epsilon$.
With probability $\ge \bigl(1-\frac{\log
  (2/\delta)}{|\cF|^{1/3}}\bigr)^{\log |\cF|}$ it always competes
with an element at a distance at least $10\epsilon$
from $f$ and it wins all these games with probability $\ge
1-\delta/2$.
\end{itemize}
Therefore with probability $\ge1-\delta/2$ there is an 
element in $\cC$ at $\ell_1$ distance at most $100\epsilon$.
Running \textsc{Scheffe*} on this set yields a distribution
at a distance $\le 100\cdot 10\epsilon = 1000\epsilon$. The error probability is $\leq \delta$ by the union bound.
\end{proof}
%

%
%
%

\section{Lower bound}
\label{app:lower}
We first show a lower bound for a single Gaussian distribution and generalize it to mixtures.
\subsection{Single Gaussian distribution}


The proof is an application of the following version 
of Fano's inequality~\cite{CoverT06, Yu97}. It states
that we cannot simultaneously estimate all distributions
in a class using $n$ samples if they satisfy certain
conditions.

\begin{Lemma}{(Fano's Inequality)} Let $f_1, \ldots, f_{r+1}$ be a collection of 
distributions such that for any
$i\ne j$, $\lone{f_i}{ f_j}\ge \alpha$, and $KL(f_i, f_j)\le\beta$. 
Let $f$ be an estimate of the underlying distribution using $n$
\iid \ samples from one of the $f_i$'s. 
Then,
\[
\sup_i \EE [\lone{f_i}{ f}]\ge\frac{\alpha}2\Big(1- \frac{n\beta +\log 2}{\log r}\Big).
\]
\end{Lemma}


We consider  $d-$dimensional spherical Gaussians with identity covariance
matrix, with means along any coordinate restricted to $\pm\frac{c\epsilon}{\sqrt
  d}$. 
The KL divergence between two spherical Gaussians with identity covariance
matrix is the squared distance between their means. Therefore, 
any two distributions we consider have KL distance at most
\[
\beta = \sum_{i=1}^d \Big(2\frac{c\epsilon}{\sqrt  d}\Big)^2 = 4c^2\epsilon^2,
\]  
 We now consider a subset of these $2^d$
distributions to obtain a lower bound on $\alpha$. 
By the Gilbert-Varshamov bound, there exists a binary code with $\geq 2^{d/8}$
codewords of length $d$ 
and minimum distance $d/8$. 
Consider one such code. Now for each codeword, map $1\to \frac{c\epsilon}{\sqrt d}$ and  
$0\to -\frac{c\epsilon}{\sqrt d}$ to obtain a distribution in our
class. We consider this subset of $\ge 2^{d/8}$ distributions as our $f_i$'s. 

Consider any two $f_i$'s. Their means differ in at least $d/8$ coordinates. 
We show that the $\ell_1$ distance between them is $\ge c\epsilon/4$. 
Without loss of generality, let the means differ in the first $d/8$ 
coordinates, and furthermore, one of the distributions has 
means $c\epsilon/\sqrt{d}$ and the other has $-c\epsilon/\sqrt{d}$ 
in the first $d/8$ coordinates. The sum of the first $d/8$ 
coordinates is $N(c\epsilon\sqrt{d}/8, d/8)$ and $N(-c\epsilon\sqrt{d}/8, d/8)$.
The $\ell_1$ distance between these normal random variables is a 
lower bound on the $\ell_1$ distance of the original random variables. 
For small values of $c\epsilon$ the distance between the two Gaussians
is at least $\ge c\epsilon/4$. This serves as our $\alpha$.

Applying the Fano's Inequality, the $\ell_1$ error on the worst 
distribution is at least
\[
\frac{c\epsilon}{8}\Big(1- \frac{n4c^2\epsilon^2 +\log 2}{d/8}\Big),
\]
which for $c=16$ and $n<\frac d{2^{14}\epsilon^2}$ is at least $\epsilon$. 
In other words, the smallest $n$ to approximate all
spherical normal distributions to $\ell_1$ distance
at most $\epsilon$ is $>\frac d{2^{14}\epsilon^2}$.

\subsection{Mixtures of $k$ Gaussians}
We now provide a lower bound on the sample complexity of
learning mixtures of $k$ Gaussians in $d$ dimensions. We 
extend the construction for learning a single spherical 
Gaussian to mixtures of $k$ Gaussians and show a lower bound
of $\Omega (kd/\epsilon^2)$ samples. 
We will again use Fano's inequality over a class of $2^{kd/64}$
distributions as described next.

To prove the lower bound on the sample complexity of learning 
spherical Gaussians, we designed a class of $2^{d/8}$ distributions 
around the origin. 
Let $\cP\ed\{P_1, \ldots, P_{T}\}$, where
$T=2^{d/8}$, be this class. Recall that each $P_i$ is a spherical
Gaussian with unit variance. 
For a distribution $P$ over $\RR^d$ and $\classmeans\in\RR^d$, let $P+\classmeans$ be the
distribution $P$ shifted by $\classmeans$. 

We now choose $\classmeans_1, \ldots, \classmeans_k$'s \emph{extremely
well-separated}. The class of distributions we consider will be a
mixture of $k$ components, where the $j$th component is 
a distribution from $\cP$ shifted by $\classmeans_j$. Since the 
$\classmeans$'s will be well separated, we will use the results from
last section over each component. 

 For $i\in[T]$, and $j\in[k]$, $P_{ij}\ed P_i+\classmeans_j$.
Each $(i_1, \ldots, i_k)\in[T]^k$ corresponds to the mixture
\[
\frac1k(P_{i_11}+P_{i_22}+\ldots+P_{i_kk})
\]
of $k$ spherical Gaussians. We consider this class of $T^k = 2^{kd/8}$ 
distributions. 
By the Gilbert-Varshamov bound, for any $T\ge2$, there is 
a $T$-ary codes of length $k$, with minimum distance 
$\ge k/8$ and number of codewords $\ge 2^{k/8}$. 
This implies that among the $T^k=2^{dk/8}$ distributions, 
there are $2^{kd/64}$ distributions such that any two tuples $(i_1,
\ldots, i_k)$ and $(i_1',\ldots, i_k')$ corresponding to different
distributions differ in at least $k/8$ locations.

If we choose the $\classmeans$'s well separated, the 
components of any mixture distribution have very little overlap. 
For simplicity, we choose $\classmeans_j$'s satisfying 
\[
\min_{j_1\ne j_2}||\classmeans_{j_1}-\classmeans_{j_2}||_2\ge \left(\frac{2kd}{\epsilon}\right)^{100}.
\]

This implies that for $j\ne l$, $\norm{P_{ij}-P_{i'l}}_1< (\epsilon/2dk)^{10}$. 
Therefore, for two different mixture distributions,
\begin{align*}
&\norm{\frac1k(P_{i_11}+P_{i_22}+\ldots+P_{i_kk})-\frac1k(P_{i_1'1}+P_{i_2'2}+\ldots+P_{i_k'k})}_1\\
\stackrel{(a)}{\ge}&\frac 1k \sum_{j\in[k],  i_j, i_j'\in[T]}|P_{i_jj}-P_{i_j'j}|-k^2 (\epsilon/2dk)^{10}\\
\stackrel{(b)}{\ge}& \frac 18\frac{c\epsilon}4-k^2 (\epsilon/2dk)^{10}.
\end{align*}
where $(a)$ follows form the fact that two mixtures have overlap only
in the corresponding components, $(b)$ uses the fact that at least in
$k/8$ components $i_j\ne i_j'$, and then uses the lower bound from the
previous section. 

Therefore, the $\ell_1$ distance between any two 
of the $2^{kd/64}$ distributions is $\ge c_1\epsilon/32$
for $c_1$ slightly smaller than $c$. We take this as $\alpha$.

Now, to upper bound the KL divergence, we simply use the 
convexity, namely for any distributions $P_1\ldots P_k$
and $Q_1\ldots Q_k$, let $\bar{P}$ and $\bar{Q}$ be the 
mean distributions. Then, 
\[
D(\bar{P}||\bar{Q})\le \frac1k\sum_{i=1}^k D(P_i||Q_i).
\]
By the construction and from the previous section, for any $j$, 
\[
D(P_{i_jj}||P_{i_j'j})= D(P_i||P_{i'})\le 4c^2\epsilon^2.
\]
Therefore, we can take $\beta = 4c^2\epsilon^2$.

Therefore by the Fano's inequality, the $\ell_1$ error on the worst
distribution is at least 
\[
\frac{c_1\epsilon}{64}\Big(1- \frac{n4c^2\epsilon^2 +\log 2}{dk/64}\Big),
\]
which for $c_1=128, c=128.1$ and $n<\frac{dk}{8^8\epsilon^2}$ is at least $\epsilon$. 

\section{Proofs for $k$ spherical Gaussians}
We first state a simple  concentration result that helps us in other proofs.
\begin{Lemma}
\label{lem:distcconc}
Given $n$ samples from a set of Gaussian distributions, with probability $\geq 1-2\delta$,
for every pair of samples $\Xvec \sim N(\mupone,\sigma^2\II_d)$ and $\Yvec \sim N(\muptwo,\sigma^2 \II_d)$,
\begin{equation}
\label{eq:distcconc1}
  \norm{\Xvec - \Yvec}^2_2 \leq  2d\sigma^2  + 4\sigma^2\sqrt{d\log\frac{n^2}{\delta}}+ 
  \norm{\mupone-\muptwo}^2_2 + 4\sigma\norm{\mupone-\muptwo}_2\sqrt{\log\frac{n^2}{\delta}} + 4\sigma^2 \log\frac{n^2}{\delta}.
\end{equation}
and 
\begin{equation}
\label{eq:distcconc2}
  \norm{\Xvec - \Yvec}^2_2 \geq 2d\sigma^2  - 4\sigma^2\sqrt{d\log\frac{n^2}{\delta}}+ 
  \norm{\mupone-\muptwo}^2_2 -4\sigma\norm{\mupone-\muptwo}_2\sqrt{\log\frac{n^2}{\delta}}.
\end{equation}
\end{Lemma}
\begin{proof}
We prove the lower bound, the proof for the upper bound is similar and omitted.
Since $\Xvec$ and $\Yvec$ are Gaussians, $\Xvec-\Yvec$ is distributed as $N(\mupone-\muptwo,2\sigma^2)$. Rewriting $ \norm{\Xvec - \Yvec}_2$
\[
  \norm{\Xvec - \Yvec}^2_2 =  \norm{\Xvec - \Yvec -(\mupone-\muptwo)}^2_2 + \norm{\mupone-\muptwo}^2_2 + 
  2(\mupone-\muptwo)\cdot(\Xvec - \Yvec -(\mupone-\muptwo)).
\]
Let $\Zvec = \Xvec - \Yvec -(\mupone-\muptwo)$, then $\Zvec \sim N(\zerovec,2\sigma^2\II_d)$.
Therefore by Lemma~\ref{lem:tailchi}, with probability $1 - \delta/n^{2}$,
\[
 \norm{\Zvec}^2_2 \geq  2 d\sigma^2 - 4\sigma^2\sqrt{d\log\frac{n^2}{\delta}}.
\]
Furthermore $(\mupone-\muptwo) \cdot \Zvec$ is sum of 
Gaussians and hence a Gaussian distribution. It has mean $0$ and variance $2\sigma^2\norm{\mupone-\muptwo}^2_2$.
Therefore, by Lemma~\ref{lem:gaussbound} with probability $ 1- \delta/n^2$, 
\[
 (\mupone-\muptwo)\cdot \Zvec \geq - 2\sigma\norm{\mupone-\muptwo}_2\sqrt{ \log\frac{n^2}{\delta}}.
\]
By the union bound with probability $1-2\delta/n^{2}$,
\begin{align*}
  \norm{\Xvec - \Yvec}^2_2 \geq 2d\sigma^2  - 4\sigma^2\sqrt{d\log\frac{n^2}{\delta}}+ 
  \norm{\mupone-\muptwo}^2_2 -4\sigma\norm{\mupone-\muptwo}_2\sqrt{ \log\frac{n^2}{\delta}}.
\end{align*}
There are ${n \choose 2}$ pairs and the lemma follows by the union bound.
\end{proof}

\subsection{Proof of Lemma~\ref{lem:sigconc}}
\label{app:sigconc}
We show that if Equations~\eqref{eq:distcconc1} and~\eqref{eq:distcconc2} are satisfied, then the lemma holds. The error probability is that
of Lemma~\ref{lem:distcconc} and is $\leq 2 \delta$.
Since the minimum is over $k+1$ indices, at least two samples are from the same component.
Applying Equations~\eqref{eq:distcconc1} and~\eqref{eq:distcconc2} for these two samples
 \[
 2d \wqsigma^2 \leq 2d\sigma^2 + 4 \sigma^2\sqrt{d\log\frac{n^2}{\delta}} + 4 \sigma^2\log\frac{n^2}{\delta}.
 \]
Similarly by Equations~\eqref{eq:distcconc1} and~\eqref{eq:distcconc2} for any two samples $\Xvec(a),\Xvec(b)$ in $[k+1]$,
\begin{align*}
  \norm{\Xvec(a) - \Xvec(b)}^2_2 &\geq 2d\sigma^2  - 4\sigma^2\sqrt{d\log\frac{n^2}{\delta}}+ 
  \norm{\mupi-\mupj}^2_2 -4\sigma\norm{\mupi-\mupj}_2\sqrt{ \log\frac{n^2}{\delta}}\\
  & \geq 2d\sigma^2  - 4\sigma^2\sqrt{d\log\frac{n^2}{\delta}} - 4\sigma^2 \log\frac{n^2}{\delta},
\end{align*}
where the last inequality follows from the fact that $\alpha^2 - 4\alpha \beta \geq -4\beta^2$.
The result follows from the assumption that $ d > 20\log n^2/\delta$.

\subsection{Proof of Lemma~\ref{lem:singlelinkage}}
\label{app:singlelinkage}
We show that if Equations~\eqref{eq:distcconc1} and~\eqref{eq:distcconc2} are satisfied, then the lemma holds. The error probability is that
of Lemma~\ref{lem:distcconc} and is $\leq 2 \delta$. Since Equations~\eqref{eq:distcconc1} and~\eqref{eq:distcconc2} are satisfied, by the proof of Lemma~\ref{lem:sigconc},  
$ |\wqsigma^2 -\sigma^2| \leq 2.5 \sigma^2 \sqrt{\frac{\log (n^2/\delta)}{d}}$.
 If two samples $X(a)$ and $X(b)$ are from the same component, by Lemma~\ref{lem:distcconc},
 \[
  \norm{\Xvec(a) - \Xvec(b)}^2_2 \leq 2d\sigma^2 + 4 \sigma^2\sqrt{d\log\frac{n^2}{\delta}} + 4 \sigma^2 \log \frac{n^2}{\delta} \leq  2d\sigma^2 + 5 \sigma^2\sqrt{d\log\frac{n^2}{\delta}}.
 \]
 By Lemma~\ref{lem:sigconc}, the above quantity is less than $2d\wqsigma^2 +  23\wqsigma^2\sqrt{d\log\frac{n^2}{\delta}}$.
Hence all the samples from the same component are in a single cluster.

Suppose there are two samples from different components in a cluster, then
by Equations~\eqref{eq:distcconc1} and~\eqref{eq:distcconc2},
\begin{align*}
 2d\wqsigma^2 + 23 \wqsigma^2\sqrt{d\log\frac{n^2}{\delta}} \geq 2d\sigma^2  - 4\sigma^2\sqrt{d\log\frac{n^2}{\delta}}+ \norm{\mupi-\mupj}^2_2 -4\sigma\norm{\mupi-\mupj}_2\sqrt{ \log\frac{n^2}{\delta}}.
\end{align*}
Relating $\wqsigma^2$ and $\sigma^2$ using Lemma~\ref{lem:sigconc}, 
\begin{align*}
2d\sigma^2 + 40 \sigma^2 \sqrt{d\log\frac{n^2}{\delta}} \geq
 2d\sigma^2  - 4\sigma^2\sqrt{d\log\frac{n^2}{\delta}}+ \norm{\mupi-\mupj}^2_2 -4\sigma\norm{\mupi-\mupj}_2\sqrt{ \log\frac{n^2}{\delta}}.
\end{align*}
Hence $\norm{\mupi-\mupj}_2  \leq 10\sigma \bigl(d \log\frac{n^2}{\delta}\bigr)^{1/4}$. 
There are at most $k$ components; therefore, any two components within the same cluster
are at a distance  $ \leq 10k\sigma \bigl(d \log\frac{n^2}{\delta}\bigr)^{1/4}$.

\subsection{Proof of Lemma~\ref{lem:reccluster}}
\label{app:reccluster}
The proof is involved and we show it in steps. We first show few concentration bounds which we use later to argue that the samples are clusterable when 
the sample covariance matrix has a large eigenvalue. 
Let $\wwi$ be the fraction of samples from component $i$. Let $\wmupi$ be the empirical average of samples from $\pveci$.
Let $\wmuave (C)$ be the empirical average of samples in cluster $C$. If $C$ is the entire set of samples we use $\wmuave$ instead of $\wmuave(C)$.
 We first show a concentration inequality that we use in rest of the calculations.
 
 \begin{Lemma}
  \label{lem:aux0}
   Given $n$ samples from a $k$-component Gaussian mixture with probability $\geq 1- 2\delta$, for every component $i$
 \begin{equation}
 \label{eq:aux0}
     \norm{\wmupi-\mupi}^2_2 \leq \biggl( d+ 3\sqrt{d \log \frac{2k}{\delta}} \biggr) \frac{\sigma^2}{n\wwi}  \text{ and } 
 |\wwi - \wi| \leq \sqrt{\frac{2\wi \log \frac{2k}{\delta}}{n}} + \frac{2}{3}\frac{\log \frac{2k}{\delta}}{n}.
\end{equation}
 \end{Lemma}
\begin{proof}
 Since $\wmupi-\mupi$ is distributed $N(0,\sigma^2\II_d/n\wwi)$, by
Lemma~\ref{lem:tailchi} with probability $\geq 1- \delta/k$,
 \begin{align*}
\norm{\wmupi-\mupi}^2_2 
&\leq   \biggl(d + 2\sqrt{d \log \frac{2k}{\delta}} + 2\log \frac{2k}{\delta} \biggr)\frac{ \sigma^2}{n\wwi} \leq  \biggl( d+ 3\sqrt{d \log \frac{2k}{\delta}} \biggr) \frac{\sigma^2}{n\wwi}.
\end{align*}
The second inequality uses the fact that $d \geq 20 \log n^2/\delta$.
For bounding the weights, observe that by Lemma~\ref{lem:eqchern} with probability $\geq 1-\delta/k$,
\[
 |\wwi - \wi| \leq \sqrt{\frac{2\wi \log 2k/\delta}{n}} + \frac{2}{3}\frac{ \log 2k/\delta}{n}.
\]
By the union bound the error probability is $\leq 2k \delta/2k =\delta$.
\end{proof}
A simple application of triangle inequality yields the following lemma.
\begin{Lemma}
\label{lem:aux1}
 Given $n$ samples from a $k$-component Gaussian mixture if Equation~\eqref{eq:aux0} holds, then
 \[
  \norm{\sum^k_{i=1} \wwi(\wmupi-\mupi)(\wmupi-\mupi)^t} \leq \biggl( d+ 3\sqrt{d \log \frac{2k}{\delta}} \biggr) \frac{k\sigma^2}{n}.
 \]
\end{Lemma}
\begin{Lemma}
\label{lem:aux2}
Given $n$ samples from a $k$-component Gaussian mixture, if Equation~\eqref{eq:aux0} holds and the maximum distance between two components is $\leq 10k\sigma \bigl(d \log \frac{n^2}{\delta}\bigr)^{1/4}$, then
$
 \norm{\wmuave-\muave)}_2 \leq c \sigma\sqrt{\frac{d k \log \frac{n^2}{\delta}}{n}},
$
for a constant $c$.
\end{Lemma}
\begin{proof}
Observe that
\begin{equation}
\label{eq6}
 \wmuave-\muave  = \sum^k_{i=1} \wwi \wmupi - \wi \mupi = \sum^k_{i=1} \wwi (\wmupi - \mupi) +(\wwi - \wi) \mupi 
 = \sum^k_{i=1} \wwi (\wmupi - \mupi) +(\wwi - \wi) (\mupi -\muave).
\end{equation}
Hence by Equation~\eqref{eq:aux0} and the fact that the maximum distance between two components is $\leq 10k\sigma \bigl(d \log \frac{n^2}{\delta}\bigr)^{1/4}$,
\begin{align*}
 \norm{\wmuave-\muave}_2
 \leq \sum^k_{i=1} \wwi  \sqrt{\biggl( d+ 3\sqrt{d \log \frac{2k}{\delta}} \biggr)} \frac{\sigma}{\sqrt{n\wwi}} + \biggl(  \sqrt{\frac{2\wi \log 2k/\delta}{n}} + \frac{2}{3}\frac{ \log 2k/\delta}{n}\biggr) 
 10k \biggl(d \log \frac{n^2}{\delta} \biggr)^{1/4}\sigma.
\end{align*}
For $n \geq d \geq \max(k^4,20\log n^2/\delta,1000)$, we get the above term is $\leq c\sqrt{\frac{kd\log n^2/\delta}{n}}\sigma$, for some constant $c$.
\end{proof}
We now make a simple observation on covariance matrices.
\begin{Lemma}
\label{lem:aux9}
Given $n$ samples from a $k$-component mixture, 
\begin{align*}	
& \norm{ \sum^k_{i=1} \wwi(\wmupi-\wmuave)(\wmupi-\wmuave)^t - \sum^k_{i=1} \wwi(\mupi-\muave)(\mupi-\muave)^t} \\
& \leq 2\norm{\wmuave-\muave}^2_2+  \sum^k_{i=1}  2\wwi \norm{\wmupi-\mupi}^2_2 + 
 2 \left(  \sqrt{k}\norm{\wmuave-\muave}_2 + \sum^k_{i=1}\sqrt{\wwi}\norm{\wmupi-\mupi}_2\right)  \max_j\sqrt{\wwj}\norm{\mupj-\muave}_2.
\end{align*}
\end{Lemma}
\begin{proof}
 Observe that for any two vectors $\uvec$ and $\vvec$,
 \begin{align*}
  \uvec \uvec^t - \vvec \vvec^t = \uvec(\uvec^t - \vvec^t) + (\uvec-\vvec)\vvec^t = (\uvec-\vvec)(\uvec-\vvec)^t + \vvec(\uvec-\vvec)^t+ (\uvec-\vvec)\vvec^t.
 \end{align*}
 Hence by triangle inequality,
 \begin{align*}
  \norm{\uvec \uvec^t - \vvec \vvec^t}  \leq \norm{\uvec-\vvec}^2_2 + 2 \norm{\vvec}_2\norm{\uvec-\vvec}_2.
 \end{align*}
 Applying the above observation to $\uvec = \wmupi-\wmuave$ and $\vvec = \mupi-\muave$, we get
 \begin{align*}
&   \sum^k_{i=1} \wwi \norm{ (\wmupi-\wmuave)(\wmupi-\wmuave)^t - (\mupi-\muave)(\mupi-\muave)^t} \\
&    \leq  \sum^k_{i=1} \left( \wwi \norm{\wmupi-\wmuave-\mupi-\muave}^2_2 +  2 \sqrt{\wwi}\norm{\mupi-\muave}_2\sqrt{\wwi}\norm{\wmupi-\wmuave-\mupi-\muave}_2 \right)\\
&  \leq  \sum^k_{i=1} \left(2 \wwi\norm{\wmupi-\mupi}^2_2 + 2 \wwi\norm{\wmuave-\muave}^2_2 +
    2 \max_j\sqrt{\wwj}\norm{\mupj-\muave}_2\left(\sqrt{\wwi}\norm{\wmupi-\mupi}_2 + \sqrt{\wwi}\norm{\wmuave-\muave}_2 \right) \right)\\ 
& \leq 2\norm{\wmuave-\muave}^2_2+  \sum^k_{i=1}  2\wwi \norm{\wmupi-\mupi}^2_2 + 
 2 \left(  \sqrt{k}\norm{\wmuave-\muave}_2 + \sum^k_{i=1}\sqrt{\wwi}\norm{\wmupi-\mupi}_2 \right) \max_j\sqrt{\wwj}\norm{\mupj-\muave}_2.
 \end{align*}
 The lemma follows from triangle inequality.
\end{proof}
The following lemma immediately follows from Lemmas~\ref{lem:aux2} and~\ref{lem:aux9}.
\begin{Lemma}
\label{lem:aux20}
Given $n$ samples from a $k$-component Gaussian mixture, if Equation~\eqref{eq:aux0} and the maximum distance between two components is $\leq 10k\sigma \bigl(d \log \frac{n^2}{\delta}\bigr)^{1/4}$, then
\begin{align*}	
\norm{ \sum^k_{i=1} \wwi(\wmupi-\wmuave)(\wmupi-\wmuave)^t - \sum^k_{i=1} \wwi(\mupi-\muave)(\mupi-\muave)^t}   
\leq\frac{ c \sigma^2 dk^2\log \frac{n^2}{\delta}}{n} + c \sigma\sqrt{\frac{dk^2\log \frac{n^2}{\delta}}{n}} \max_{i}\sqrt{ \wwi }\norm{\mupi-\muave}_2,
\end{align*}
for a constant $c$.
\end{Lemma}
\begin{Lemma}
\label{lem:exp}
For a set of samples $\Xvec(1), \ldots \Xvec(n)$ from a $k$-component mixture,
\begin{align*}
\sum^n_{i=1} \frac{(\Xvec(i)-\wmuave)(\Xvec(i)-\wmuave)^t}{n}  =  \sum^k_{i=1} \wwi(\wmupi-\wmuave)(\wmupi-\wmuave)^t -  \wwi(\wmupi-\mupi)(\wmupi-\mupi)^t + 
  \sum_{j | \Xvec(j) \sim p_i}  \frac{(\Xvec(j)-\mupi)(\Xvec(j)-\mupi)^t}{n}.
\end{align*}
where $\wwi$ and $\wmupi$ are the empirical weights and averages of components $i$ and $\wmuave = \frac{1}{n}\sum^n_{i=1} \Xvec_i$.
\end{Lemma}
\begin{proof}
The given expression can be rewritten as
\[
\frac{1}{n}\sum^n_{i=1} (\Xvec(i)-\wmuave)(\Xvec(i)-\wmuave)^t  =
 \sum^k_{i=1} \wwi \sum_{j | \Xvec(j) \sim p_i}  \frac{1}{n\wwi} \Xvec(j)-\wmuave)(\Xvec(j)-\wmuave)^t.
\]
First observe that for any set of points $x_i$ and their average $\hat{x}$ and any value $a$,
\[
\sum_{i} (x_i-a)^2= \sum_{i} (x_i-\hat{x})^2 + (\hat{x}-a)^2.
\]
Hence for samples from a component $i$,
\begin{align*}
&\sum_{j | \Xvec(j) \sim p_i}  \frac{1}{n\wwi} (\Xvec(j)-\wmuave)(\Xvec(j)-\wmuave)^t \\
&= \sum_{j | \Xvec(j) \sim p_i}  \frac{1}{n\wwi} (\wmupi-\wmuave)(\wmupi-\wmuave)^t  + \sum_{j | \Xvec(j) \sim p_i}  \frac{1}{n\wwi}  (\Xvec(j)-\wmupi)(\Xvec(j)-\wmupi)^t \\
& = (\wmupi-\wmuave)(\wmupi-\wmuave)^t+ \sum_{j | \Xvec(j) \sim p_i}  \frac{1}{n\wwi}(\Xvec(j)-\wmupi)(\Xvec(j)-\wmupi)^t \\
& = (\wmupi-\wmuave)(\wmupi-\wmuave)^t+ \sum_{j | \Xvec(j) \sim p_i}  \frac{1}{n\wwi}(\Xvec(j)-\mupi)(\Xvec(j)-\mupi)^t
 - (\wmupi-\mupi)(\wmupi-\mupi)^t.
\end{align*}
Summing over all components results in the lemma.
\end{proof}
We now bound the error in estimating the eigenvalue of the covariance matrix.
\begin{Lemma}
\label{lem:covconc}
 Given $\Xvec(1), \ldots \Xvec(n)$, $n$ samples from a $k$-component Gaussian mixture, if Equations~\eqref{eq:distcconc1},~\eqref{eq:distcconc2}, and~\eqref{eq:aux0} hold, then 
with probability $\geq 1-2\delta$, 
\begin{equation}
\begin{aligned}
\label{eq:covconc}
&  \norm{\frac{1}{n}\sum^n_{i=1} (\Xvec(i)-\wmuave)(\Xvec(i)-\wmuave)^t -\wsigma^2 \II_d - \sum^k_{i=1} \wwi(\mupi-\muave)(\mupi-\muave)^t} 
  \\ 
  &\leq 
c(n) \ed c \sigma^2  \sqrt{\frac{d \log \frac{n^2}{\delta}}{n}}+ 
c \sigma^2 \frac{dk^2\log \frac{n^2}{\delta}}{n} + c \sigma\sqrt{\frac{dk^2\log \frac{n^2}{\delta}}{n}} \max_{i}\sqrt{ \wwi }\norm{\mupi-\muave}_2,
\end{aligned}
\end{equation}
for a constant $c$.
\end{Lemma}
\begin{proof}
Since Equations~\eqref{eq:distcconc1},~\eqref{eq:distcconc2}, and~\eqref{eq:aux0} hold, conditions in Lemmas~\ref{lem:aux2} and~\ref{lem:aux20} are satisfied. By Lemma~\ref{lem:aux20},
\begin{align*}	
\norm{ \sum^k_{i=1} \wwi(\wmupi-\wmuave)(\wmupi-\wmuave)^t - \sum^k_{i=1} \wwi(\mupi-\muave)(\mupi-\muave)^t}   
= \cO \left(  \sigma^2 \frac{dk^2\log \frac{n^2}{\delta}}{n} + \sigma\sqrt{\frac{dk^2\log \frac{n^2}{\delta}}{n}} \max_{i}\sqrt{ \wwi }\norm{\mupi-\muave}_2 \right).
\end{align*}
Hence it remains to show,
\[
  \norm{\frac{1}{n}\sum^n_{i=1} (\Xvec(i)-\wmuave)(\Xvec(i)-\wmuave)^t - \sum^k_{i=1} \wwi(\wmupi-\wmuave)(\wmupi-\wmuave)^t} = \cO \left( \sqrt{\frac{kd \log \frac{5k^2}{\delta}}{n}} \sigma^2 \right) .
 \]
By Lemma~\ref{lem:exp}, the covariance matrix can be rewritten as
\begin{equation}
\label{eq5}
 \sum^k_{i=1} \wwi(\wmupi-\wmuave)(\wmupi-\wmuave)^t -  \wwi(\wmupi-\mupi)(\wmupi-\mupi)^t + 
  \sum^k_{i=1}  \sum_{j | \Xvec(j) \sim p_i}  \frac{1}{n}(\Xvec(j)-\mupi)(\Xvec(j)-\mupi)^t  -\wqsigma^2 \II_d.
  \end{equation}
We now bound the norms of second and third terms in the above equation.
Consider the third term, $  \sum^k_{i=1}  \sum_{j | \Xvec(j) \sim p_i}  \frac{1}{n}(\Xvec(j)-\mupi)(\Xvec(j)-\mupi)^t$.
Conditioned on the fact that $\Xvec(j) \sim p_i$, $\Xvec(j)-\mupi$ is distributed $N(0,\sigma^2\II_d)$,
therefore by Lemma~\ref{lem:matconc} and Lemma~\ref{lem:sigconc} ,with probability $\geq 1-2\delta$,
\[
\norm{\sum^k_{i=1}  \sum_{j | \Xvec(j) \sim p_i}  \frac{1}{n}(\Xvec(j)-\mupi)(\Xvec(j)-\mupi)^t - \wqsigma^2 \II_d}
\leq  c'  \sqrt{\frac{d \log \frac{2d}{\delta}}{n}}\sigma^2 + 2.5\sigma^2 \sqrt{\frac{\log\frac{ n^2}{\delta}}{d}}.
\]
The second term in Equation~\eqref{eq5} is bounded by Lemma~\ref{lem:aux1}.
Hence together with the fact that $d \geq 20 \log n^2/\delta$ we get that with probability $\geq 1-2 \delta$, the second and third terms are bounded by 
$
\cO \left( \sigma^2\sqrt{\frac{d k}{n} \log \frac{n^2}{\delta}} \right).
$
\end{proof}

\begin{Lemma}
\label{lem:bound}
Let $\uvec$ be the largest eigenvector of the sample covariance matrix and $n \geq 
c \cdot dk^2 \log \frac{n^2}{\delta}$.
If $\max_i \sqrt{\wwi}  \norm{\mupi-\muave}_2 = \alpha \sigma$ and Equation~\eqref{eq:covconc} holds,
then there exists $i$ such that $|\uvec \cdot (\mupi - \muave)| \geq \sigma(\alpha - 1 - 1/\alpha)/\sqrt{k}$.
\end{Lemma}
\begin{proof}
 Observe that $\norm{\sum_j w_j \vvec_j\vvec^t_j  }  \geq \norm{\sum_j w_j \vvec_j \vvec^t_j \frac{\vvec_i}{\norm{\vvec_i}}}_2 \geq w_i \norm{\vvec_i}^2_2$. Therefore
\begin{align*}
  \norm{ \sum^k_{i=1} \wwi(\mupi-\muave)(\mupi-\muave)^t} 
   \geq   \norm{ \sum^k_{j=1} \wwj(\mupj-\muave)(\mupj-\muave)^t (\mupi-\muave)/\norm{\mupi-\muave}}_2
  \geq \alpha^2 \sigma^2.
\end{align*}
Hence by Lemma~\ref{lem:covconc} and the triangle inequality, the largest eigenvalue of the sample-covariance matrix
is $\geq \alpha^2 \sigma^2 -  c(n)$.
Similarly by applying Lemma~\ref{lem:covconc} again we get,$
   \norm{ \sum^k_{i=1} \wwi(\mupi-\muave)(\mupi-\muave)^t \uvec}_2  \geq \alpha^2 \sigma^2 -   2c(n).
$
By triangle inequality and Cauchy-Schwartz inequality,
\begin{align*}
    \norm{ \sum^k_{i=1} \wwi(\mupi-\muave)(\mupi-\muave)^t \uvec}_2 &\leq   \sum^k_{i=1}\norm{  \wwi(\mupi-\muave)(\mupi-\muave)^t \uvec}_2 \\
    & \leq  \sum^k_{i=1}\wwi \norm{  (\mupi-\muave)}_2 \max_j |(\mupj-\muave)\cdot \uvec| \\
    & \leq \sqrt{\sum^k_{i=1}\wwi \norm{  (\mupi-\muave)}^2_2}  \max_j|(\mupj-\muave)\cdot \uvec| \\
    & \leq \sqrt{k} \alpha\sigma \max_j |(\mupj-\muave)\cdot \uvec|.
\end{align*}
Hence $\sqrt{k} \alpha\sigma \max_i |(\mupi-\muave)\cdot \uvec| \geq \alpha^2 \sigma^2 -  2c(n)$. The lemma follows by substituting the bound on $n$ in $c(n)$.
\end{proof}
We now make a simple observation on Gaussian mixtures.
\begin{Fact}
\label{fac:cluster}
The samples from a subset of components $A$ of the Gaussian mixture
 are distributed according to a Gaussian mixture of components $A$ with weights
 being $\wi' = \wi/(\sum_{j \in A} \wj)$.
\end{Fact}
We now prove Lemma~\ref{lem:reccluster}.
\begin{proof}[Proof of Lemma~\ref{lem:reccluster}]
Observe that we run the recursive clustering at most $n$ times.
At every step, the underlying distribution within a cluster is a Gaussian mixture.
Let Equations~\eqref{eq:distcconc1},~\eqref{eq:distcconc2} hold with probability $1-2\delta$.
Let Equations~\eqref{eq:aux0}~\eqref{eq:covconc} all hold with probability $\geq 1-\delta'$, where $\delta' = \delta/2n$ at each of $n$ steps.
By the union bound 
the total error is $\leq 2\delta + \delta'\cdot 2n \leq 3\delta$.
Since Equations~\eqref{eq:distcconc1},~\eqref{eq:distcconc2} holds, the conditions of Lemmas~\ref{lem:sigconc} and~\ref{lem:singlelinkage} hold.
Furthermore it can be shown that discarding at most $n\epsilon/4k$ samples at each step does not affect
the calculations.

We first show that if $\sqrt{\wi}\norm{\mupi-\muave(C)}_2 \geq  25  \sqrt{k^{3}\log (n^3/\delta)}\sigma$, then the algorithm 
gets into the loop. Let $\wi'$ be the weight of the component within the cluster and $n' \geq n\epsilon/5k$ be the number of samples in the cluster.
Let $\alpha = 25  \sqrt{k^{3}\log (n^3/\delta)}$. 
By Fact~\ref{fac:cluster}, the components
in cluster $C$ have weight $\wi' \geq \wi$.
Hence $\sqrt{\wi'}\norm{\mupi-\muave(C)}_2 \geq \alpha \sigma$.
Since $\sqrt{\wi'}\norm{\mupi-\muave(C)}_2 \geq \alpha \sigma$, and by Lemma~\ref{lem:singlelinkage} 
$\norm{\mupi-\muave(C)} \leq 10k \sigma (d\log n^2/\delta)^{1/4}$,
we have $\wi' \geq \alpha^2 /(100k^2 \sqrt{d\log n^2/\delta})$. 
Hence by lemma~\ref{lem:aux0}, $\wi' \geq \wi/2$
and $\sqrt{\wwi'}\norm{\mupi-\muave(C)}_2 \geq \alpha\sigma/\sqrt{2}$.
Hence by Lemma~\ref{lem:covconc} and triangle inequality the largest eigenvalue of $S(C)$ is
\[
\geq \alpha^2 \sigma^2/2 -   c(n') \\
\geq \alpha^2 \sigma^2/4 \geq \alpha^2 \wsigma^2/8 \geq 12 \wsigma^2 k^3\log n^2 /\delta' =12 \wsigma^2 k^3\log n^3 /\delta .
\]
Therefore the algorithm gets into the loop.

If $n' \geq n\epsilon/8k^2 \geq c \cdot dk^2 \log\frac{n^3}{\delta}$, then
 by Lemma~\ref{lem:bound},  there exists a component $i$ such that
$|\uvec \cdot (\mupi - \muave(C))| \geq \sigma(\alpha/\sqrt{2}- 1 - \sqrt{2}/\alpha)/\sqrt{k}$, where $\uvec$ is the top eigenvector of the first $n\epsilon/4k^2$ samples.

Observe that $\sum_{i \in C} \wi \uvec \cdot (\mupi - \muave(C)) = 0$ and 
$\max_i|\uvec \cdot (\mupi - \muave(C))| \geq \sigma(\alpha/\sqrt{2}- 1 - \sqrt{2}/\alpha)/\sqrt{k}$.
Let $\mupi$ be sorted according to their values of $\uvec \cdot (\mupi - \muave(C))$,
then 
\[
\max_i |\uvec \cdot (\mupi -\mupione)| \geq \sigma\frac{\alpha/\sqrt{2} - 1 - \sqrt{2}/\alpha}{k^{3/2}} \geq 12 \sigma \sqrt{\log \frac{n^3}{\delta}} \geq  9 \wsigma \sqrt{\log \frac{n^3}{\delta}},       
\]
where the last inequality follows from Lemma~\ref{lem:sigconc} and the fact that $d \geq 20 \log n^2/\delta$.
For a sample from component $\pveci$, similar to the proof of Lemma~\ref{lem:singlelinkage},
by Lemma~\ref{lem:gaussbound}, with probability  $\geq 1- \delta/n^2k$,
\[
\norm{u\cdot (\Xvec(i) - \mupi)} \leq \sigma\sqrt{2 \log (n^2k/\delta)}_2
\leq 2\wqsigma  \sqrt{\log (n^2k/\delta)},
\]
where the second inequality follows from Lemma~\ref{lem:sigconc}.
Since there are two components that are far apart by $\geq 9 \wsigma \sqrt{\log \frac{n^2}{\delta}}\wsigma$
and the maximum distance between a sample and its mean is  $\leq 2 \wqsigma \sqrt{\log (n^2k/\delta)}$
and the algorithm divides into at-least two non-empty 
clusters such that no two samples from the same distribution are clustered into two clusters.

For the second part observe that by the above concentration on $\uvec$, no two samples from the same component
are clustered differently irrespective of the mean separation. Note that we are using the fact that each sample
 is clustered at most $2k$ times to get the bound on
the error probability. The total error probability by the union bound is $\leq 4\delta$.
\end{proof}
%
\subsection{Proof of Lemma~\ref{lem:spancluster}}
\label{app:spancluster}
We show that if the conclusions in Lemmas~\ref{lem:reccluster} and~\ref{lem:aux0} holds, then the lemma is satisfied. We also assume that 
the conclusions in Lemma~\ref{lem:covconc} holds for all the clusters with error probability $\delta' = \delta/k$.
By the union bound the total error probability is $\leq 7 \delta$.

By Lemma~\ref{lem:reccluster} all the components within each cluster satisfy $\sqrt{\wi}\norm{\mupi - \muave(C)}_2 \leq 25  \sigma \sqrt{k^3\log (n^3/\delta)}$.
Let $n \geq c \cdot dk^{9} \epsilon^{-4}\log^2 d/\delta$. 
For notational convenience let $S(C) = \frac{1}{|C|} \sum^{|C|}_{i=1} (\Xvec(i) - \muave(C))(\Xvec(i) - \muave(C))^t -\wsigma^2 \II_d  $.
Therefore by Lemma~\ref{lem:covconc} for large enough $c$, 
\[
\norm{ S(C)- \frac{n}{|C|}\sum_{i \in C} \wwi(\mupi-\muave(C))(\mupi- \muave(C))^t} \leq \frac{\epsilon^{2}\sigma^2}{1000k^2} \frac{n}{|C|}.
\]
Let $\vvec_1,\vvec_2, \ldots \vvec_{k-1}$ be the top eigenvectors of $\frac{1}{|C|}\sum_{i \in C} \wi(\mupi-\muave(C))(\mupi- \muave(C))^t$.
Let  $\eta_i = \sqrt{\wwi'}\norm{\mupi-\muave(C)}_2 = \sqrt{\wwi}\sqrt{\frac{n}{|C|}}\norm{\mupi-\muave(C)}_2$.
Let $\dvec_i = \frac{\mupi-\muave(C))}{\norm{(\mupi-\muave(C))}_2}$.
Therefore,
\[
\sum_{i \in C} \frac{n}{|C|}\sum_{i \in C} \wwi(\mupi-\muave(C))(\mupi- \muave(C))^t = \sum_{i \in C} \eta^2_i \dvec_i \dvec^t_i.
\]
Hence by Lemma~\ref{lem:matspan}, the projection of $\dvec_i$ on the space orthogonal to top $k-1$ eigenvectors 
of $S(C)$ is
\[
\leq  \sqrt{ \frac{\epsilon^{2}\sigma^2}{1000k^2} \frac{n}{|C|}} \frac{1}{\eta_i} \leq \frac{\epsilon \sigma}{16\sqrt{\wwi} \norm{\mupi-\muave(C)}_2k} 
\leq \frac{\epsilon \sigma}{8\sqrt{2}\sqrt{\wi} \norm{\mupi-\muave(C)}_2k}.
\]
The last inequality follows from the bound on $\wwi$ in Lemma~\ref{lem:aux0}.

\subsection{Proof of Theorem~\ref{thm:ksphere}}
\label{app:thmksphere}
We show that the theorem holds if the conclusions in Lemmas~\ref{lem:spancluster} and~\ref{lem:aux2} holds with error
probability $\delta' = \delta/k$.
Since in the proof of Lemma~\ref{lem:spancluster}, the probability that Lemma~\ref{lem:reccluster} holds is included,
Lemma~\ref{lem:reccluster} also holds with the same probability. Since there
are at most $k$ clusters, by the union bound the total error probability is $\leq 9\delta$.

For every component $i$, we show that there is a choice of mean vector
and weight 
in the search step such that $\wi\lone{\pveci}{\wpveci} \leq \epsilon/2k$ and $|\wi-\wwi| \leq \epsilon/4k$.
That would imply that there is a $\wfvec$ during the search such that
\begin{align*}
\lone{\fvec}{\wfvec} & \leq \sum_{C} \sum_{i \in C} \wi \lone{\pveci}{\wpveci} +2\sum^{k-1}_{i=1} |\wi -\wwi| \leq \frac{\epsilon}{2k}+\frac{\epsilon}{2k} = \epsilon.
\end{align*}
Since the weights are gridded by $\epsilon/4k$, there exists a $\wwi$ such that  $|\wi-\wwi| \leq \epsilon/4k$.
We now show that there exists a choice of mean vector such that $\wi\lone{\pveci}{\wpveci} \leq \epsilon/2k$. 
%
Note that if a component has weight $\leq \epsilon/4k$, the above inequality follows immediately.
Therefore we only look at those components with $\wi \geq \epsilon/4k$, by Lemma~\ref{lem:aux0},
for such components $\wwi \geq \epsilon/5k$ and therefore we only look at clusters such that $|C| \geq n \epsilon/5k$.
By Lemmas~\ref{lem:l1bha} and for any $i$,
\begin{align*}
 \lone{\pveci}{\wpveci}^2 
& \leq 2 \sum^d_{j=1} \frac{(\mupij -\wmupij)^2}{\sigma^2} +8d \frac{(\sigma^2-\wsigma^2)^2}{\sigma^4}.
\end{align*}
Note that since we are discarding at most $n\epsilon/8k^2$ random samples at each step. A total number of $\leq n\epsilon/8k$
random samples are discarded. It can be shown that this does not affect our calculations 
and we ignore it in this proof.
By Lemma~\ref{lem:sigconc}, the first estimate of $\sigma^2$ satisfies 
$|\wqsigma^2- \sigma^2| \leq 2.5\sigma^2 \sqrt{\log n^2/\delta}$. Hence while searching over values
of $\wsigma^2$, there exist one such that $|\sigma'^2 - \sigma^2| \leq \epsilon\sigma^2/\sqrt{64dk^2}$.
Hence,
\begin{align*}
 \lone{\pveci}{\wpveci}^2 
& \leq 2 \frac{\norm{\mupi -\wmupi}^2_2}{\sigma^2} + \frac{\epsilon^2}{8k^2}.
\end{align*}
Therefore if we show that there is a mean vector $\wmupi$ during the search such that $\norm{\mupi-\wmupi}_2\leq \epsilon\sigma/\sqrt{16k^2\wwi}$,
that would prove the Lemma. By triangle inequality,
\[
\norm{\mupi - \wmupi}_2 \leq \norm{\muave(C)-\wmuave(C)}_2 + \norm{\mupi-\muave(C) - (\wmupi - \wmuave(C))}_2. 
\]
By Lemma~\ref{lem:aux2} for large enough $n$,
\[
\norm{\muave(C)-\wmuave(C)}_2 \leq c\sigma \sqrt{\frac{dk\log^2 n^2/\delta}{|C|}} \leq  \frac{\epsilon \sigma}{8k\sqrt{\wi}}.
\]
The second inequality follows from the bound on $n$ and the fact that $|C| \geq n \wwi$. Since $\wi \geq \epsilon/4k$, 
by Lemma~\ref{lem:aux0}, $\wwi \geq \wi/2$, we have 
\[
 \norm{\mupi - \wmupi}_2 \leq \norm{\mupi-\muave(C) -(\wmupi - \wmuave(C))}_2 + \frac{\epsilon \sigma}{8k\sqrt{\wi}}. 
\]
Let $\uvec_1 \ldots \uvec_{k-1}$ are the top eigenvectors the
sample covariance matrix of cluster $C$.
We now prove that during the search, there is a vector of the form $\sum^{k-1}_{j=1} g_j \epsilon_g \wsigma \uvec_j$ such that $\norm{\mupi-\muave(C) - 
\sum^{k-1}_{j=1} g_j \epsilon_g \wsigma \uvec_j}_2 \leq \frac{\epsilon \sigma}{8k\sqrt{\wi}}$,
during the search, thus proving the lemma. Let $\eta_i  = \sqrt{\wi} \norm{\mupi-\muave(C)}_2$.
By Lemma~\ref{lem:spancluster}, there are set of coefficients $\alpha_i$ such that
\[
\frac{\mupi-\muave(C)}{\norm{\mupi -\muave(C)}}_2 = \sum^{k-1}_{j=1} \alpha_j \uvec_j + \sqrt{1-\norm{\alpha}^2} \uvec',
\]
where $\uvec'$ is perpendicular to $\uvec_1 \ldots \uvec_{k-1}$ and $\sqrt{1-\norm{\alpha}^2} \leq \epsilon \sigma /(8\sqrt{2}\eta_i k)$.
Hence, we have 
\[
\mupi-\muave(C) = \sum^{k-1}_{j=1} \norm{\mupi -\muave(C)}_2 \alpha_j \uvec_j + \norm{\mupi -\muave(C)}_2 \sqrt{1-\norm{\alpha}^2_2} \uvec',
\]
Since $\wi \geq \epsilon/4k$ and by Lemma~\ref{lem:reccluster}, $\eta_i \leq 25\sqrt{k^{3}}\sigma \log (n^3/\delta)$, and $\norm{\mupi-\muave(C)}_2 \leq  100\sqrt{k^{4}\epsilon^{-1}} \sigma \log (n^3/\delta)$.
Therefore $\exists g_j$ such that $|g_j\wsigma-\alpha_j| \leq \epsilon_g \wsigma$ on each eigenvector. Hence,
\begin{align*}
\wi \norm{\mupi-\muave(C)  -  \sum^{k-1}_{i=1} g_j \epsilon_g \wsigma \uvec_j}^2_2 & \leq \wi k \epsilon^2_g \wsigma^2+  \wi\norm{\mupi -\muave(C)}^2_2 (1-\norm{\alpha}^2) \\
& \leq  k \epsilon^2_g \wsigma^2 +  \eta^2_i \frac{\epsilon^2 \sigma^2}{128 \eta^2_i k^2} \\
& \leq \frac{\epsilon^2\sigma^2}{128k^2} + \frac{\epsilon^2\sigma^2}{128k^2} \leq \frac{\epsilon^2\sigma^2}{64k^2}.
\end{align*}
The last inequality follows by Lemma~\ref{lem:sigconc} and the fact that $\epsilon_g \leq \epsilon/16k^{3/2}$, and hence the theorem.
The run time can be easily computed by retracing the steps of the algorithm and using an efficient implementation of single-linkage.
%
%
%

\end{document}